\documentclass{aims} 
\usepackage{amsmath}
\usepackage{paralist}
\usepackage[misc]{ifsym}
\usepackage{epsfig} 
\usepackage{epstopdf} 
\usepackage{amssymb,bbm,booktabs}
\usepackage{algorithm,algpseudocode}
\usepackage{subcaption} 
\usepackage{xurl} 
\usepackage[colorlinks=true]{hyperref}
\hypersetup{urlcolor=blue, citecolor=red}
\allowdisplaybreaks

\def\currentvolume{X}
 \def\currentissue{X}
  \def\currentyear{200X}
   \def\currentmonth{XX}
    \def\ppages{X-XX}
     \def\DOI{10.3934/xx.xxxxxxx}

\newtheorem{theorem}{Theorem}[section]

\newtheorem{lemma}[theorem]{Lemma}
\newtheorem{proposition}[theorem]{Proposition}

\theoremstyle{definition}

\newcommand{\D}[1]{\mathrm{d}#1}
\DeclareMathOperator*{\argmin}{arg\,min}

\title[Superpixel-based image segmentation]
{Superpixel-based image segmentation using squared 2-Wasserstein distances}
\author[J. Huang, A. Alpers, K. Chen and N. Lei]{}
\thanks{$^*$Corresponding authors: Andreas Alpers and Ke Chen}

\begin{document}
\maketitle

\centerline{\scshape
Jisui Huang$^{{\href{mailto:jisui@liverpool.ac.uk}{\textrm{\Letter}}}1,2}$,
Andreas Alpers$^{{\href{mailto:andreas.alpers@liverpool.ac.uk}{\textrm{\Letter}}}*1}$,}
\centerline{\scshape
Ke Chen$^{{\href{mailto:k.chen@strath.ac.uk}{\textrm{\Letter}}}*3}$
and Na Lei$^{{\href{mailto:nalei@dlut.edu.cn}{\textrm{\Letter}}}4}$}

\medskip

{\footnotesize
\centerline{$^1$Department of Mathematical Sciences}
\centerline{Centre for Mathematical Imaging Techniques (CMIT)}
\centerline{University of Liverpool, Peach Street}
\centerline{Liverpool L69 7ZL, Merseyside, United Kingdom}
}

\medskip

{\footnotesize
\centerline{$^2$School of Mathematical Sciences, Capital Normal University}
\centerline{West Third Ring Road North, Haidian District, Beijing 100084, China}
}

\medskip

{\footnotesize
\centerline{$^3$Department of Mathematics and Statistics}
\centerline{University of Strathclyde, Richmond St}
\centerline{Glasgow G1 1XH, Scotland, United Kingdom}
}

\medskip

{\footnotesize
\centerline{$^4$DUT-RU International School of Information Sciences \& Engineering}
\centerline{Dalian University of Technology, Linggong Road}
\centerline{Dalian 116024, Liaoning, China}
}

\bigskip

\centerline{(Communicated by Handling Editor)}

\begin{abstract}
We present an efficient method for image segmentation
in the presence of strong intensity inhomogeneities.
The approach combines two levels of clustering:
pixels are first grouped into superpixels using
Power-SLIC, which admits a discrete optimal
transport (OT) interpretation, and these superpixels
are subsequently merged using squared
2-Wasserstein distances between their empirical
distributions. We formulate the merging process
as an adjacency-constrained hierarchical optimization
problem with a regularized objective that accounts
for the heterogeneities of previously formed regions.
We establish {\mathversion{normal}$\mathsf{NP}$}-completeness of the associated
global optimization problem, even for two final
regions and two-bin histograms. For the proposed
greedy algorithm, we prove exact recovery under
suitable histogram separation conditions and
derive an additive approximation bound for the
regularized energy. Numerical experiments demonstrate
competitive segmentation accuracy on images
exhibiting substantial intensity inhomogeneity,
while retaining computational efficiency.
\end{abstract}

\section{Introduction}

Image segmentation is a fundamental problem in image processing and computer vision, aiming to partition an image into spatially coherent regions with common features. From an inverse-problems perspective,
segmentation seeks to recover an unknown partition
of the image domain, potentially together with
region-dependent parameters, from observed image
data, often through variational formulations
balancing data fidelity and spatial regularity
\cite{aubert2006mathematical}. 

A particular challenge arises from intensity or color inhomogeneity, where pixels belonging to the same object may exhibit substantial variation. Classical region-based variational models such as the Chan--Vese model~\cite{chan2001active} approximate regions by constant mean intensities and are therefore most effective for approximately piecewise-constant images. Numerous extensions have introduced more flexible region descriptions, including spatially varying local fitting functions~\cite{wang2010efficient,li2007implicit}, bias-field models~\cite{li2011level}, and statistical models~\cite{zhang2015level}. These approaches provide increasingly flexible descriptions of image regions, but often require spatially varying quantities to be repeatedly estimated on the pixel grid.

An alternative strategy is to first decompose the image into superpixels and subsequently operate on the resulting region-adjacency graph. Superpixels provide a compact, boundary-preserving representation of the image and can substantially reduce the computational cost of subsequent segmentation. They have therefore been incorporated into level-set methods~\cite{hao2016superpixel, zhou2017superpixel, 6466882, zhou2015learning, li2019fast, xia2019ivus, liu2022superpixel, liu2012superpixel} and region-merging approaches~\cite{wang2017optimal, doi:10.1080/01431161.2017.1354266, xiang2020fast}. In the latter, however, the dissimilarity between regions is commonly based on low-order descriptors such as differences of mean color or variance. Such descriptors can be insufficient when two regions have similar low-order statistics but substantially different color distributions.
This motivates representing an image region by its empirical feature distribution rather than by a small number of summary statistics. Optimal transport (OT) provides a natural framework for comparing such distributions, since Wasserstein distances take into account both differences between the distributions and the geometry of the underlying feature space~\cite{panaretos2019statistical,villani2009optimal,peyre2019computational,friesecke2024optimal}. 

OT distances have previously been incorporated into
image segmentation in several ways. Swoboda and
Schnörr~\cite{swoboda2013variational} formulated
variational segmentation and cosegmentation models
using Wasserstein distances between appearance
distributions and derived convex relaxations of
the resulting optimization problems. Rabin and
Papadakis~\cite{rabin2015convex} developed a convex
color-image segmentation model using regularized
OT distances between histograms. Li, Li and
Cao~\cite{IPIPaper} proposed supervised and
unsupervised segmentation models based on the
$L^1$ Monge--Kantorovich distance between feature
histograms. Further Wasserstein-based segmentation
approaches include~\cite{li2013variational}.
OT has also been used directly in region-merging
approaches. Navlakha et al.~\cite{navlakha2013unsupervised}
combined an initial watershed oversegmentation with
greedy merging of adjacent regions based on Earth
Mover's Distances between intensity and texture
histograms, using a similarity criterion that also
accounts for region sizes. An earlier conference
paper by the present authors~\cite{HuangChenAlpersLei2025}
investigated OT-based image segmentation using a
variational functional and pixel-level region merging.

In contrast to these approaches, we formulate a
regularized optimization problem over admissible
merge sequences, based on squared $2$-Wasserstein
distances between empirical region distributions.
The merge costs incorporate heterogeneity values
associated with previous merges, thereby accounting
for the merging history rather than relying solely
on the similarity of the current regions. Unlike
variational OT segmentation models involving
prescribed appearance distributions, our formulation
requires no reference distributions in the
unsupervised setting.

Starting from an oversegmentation into superpixels,
the regions form the vertices of a region-adjacency
graph. At each step, the greedy algorithm merges
two adjacent regions according to the regularized
merge cost. The empirical distributions of the
merged regions are combined, and their Wasserstein
distances to neighbouring regions are recomputed.
This produces a hierarchy of progressively coarser
partitions, with computations performed at the
region rather than pixel level.

We initialize the hierarchy using
Power-SLIC~\cite{fiedler2020power}, which itself
admits an OT interpretation: its superpixel
construction is related to a weight-balanced
least-squares assignment problem that can be
formulated as a discrete OT problem, whose dual
gives rise to anisotropic power diagrams;
cf.~\cite{alpersSIAMreview}.
OT thus enters our approach in two complementary
roles: through the assignment structure underlying
the initial superpixel partition and through the
Wasserstein geometry used to compare evolving
region distributions.

The main contributions of this work are:
\begin{itemize}
\item We formulate superpixel-based image segmentation
as an adjacency-constrained hierarchical optimization
problem involving squared $2$-Wasserstein distances
between empirical region distributions. The associated
regularized merge criterion incorporates the
heterogeneities of previously formed regions,
distinguishing the model from similarity-based
agglomerative approaches.

\item We establish $\mathsf{NP}$-completeness of the global
SP optimization problem for two final regions,
even for two-bin histograms and region-adjacency
graphs of treewidth at most three.
This demonstrates the computational difficulty
of globally optimizing the proposed merging model.

\item We establish an exact-recovery result for the
greedy algorithm under a quantitative separation
condition on two classes of region histograms.
Under the same assumptions, we derive an additive
approximation bound for the regularized energy.

\item We develop efficient greedy algorithms for
the proposed model, including unsupervised,
automatic region-number, and marker-based variants.
The algorithms operate on a region-adjacency graph
initialized by Power-SLIC, providing a two-level
framework in which optimal transport enters both
the initial superpixel construction and the
subsequent comparison of region distributions.

\item We analyse the running time of the proposed
algorithm and compare its numerical performance
with variational, superpixel-based, and
deep-learning segmentation approaches. The
experiments demonstrate the benefits of
distributional Wasserstein comparisons for
images exhibiting substantial intensity
inhomogeneity.
\end{itemize}

The remainder of the paper is organized as follows. Section~\ref{sec:Preliminaries} introduces the notation and mathematical background. Section~\ref{sec:Relevant_Models} reviews the segmentation models used for comparison. Section~\ref{sec:segmodels} introduces the proposed superpixel model, and Section~\ref{sect:SPAlgs} presents its optimization algorithms and computational analysis. Sections~\ref{sec:experiments} and~\ref{sect:results} describe the numerical experiments and results, respectively, and Section~\ref{sec:conclusion}  concludes the paper. The appendices contain the proofs of the
$\mathsf{NP}$-completeness result, the running-time bound,
and the exact-recovery and approximation guarantees.


\section{Notation and mathematical background}\label{sec:Preliminaries}

We denote by \( \mathbb{R}^d \) the \( d \)-dimensional Euclidean space with the standard Euclidean norm \( \| \cdot \| \), and by \( \mathbb{R}_{\geq 0} \) the set of nonnegative real numbers. We model an image as a function
$I:\Omega \to \mathbb{R}^c,$  with 
$\Omega\subseteq \mathbb{R}^d$ denoting the spatial domain. Gray-level images correspond to the case 
$c=1,$ whereas color images typically have $c=3.$ For $x\in \Omega,$ we write
$I(x)=(I_1(x),\dots,I_c(x))^T$ for its channel components. We will refer to elements $x \in \Omega$ as pixel locations, even when the
ambient dimension $d$ is not restricted to $2$; when no confusion can arise, we will simply write pixels.  This is a continuous model. A discretized image is a representation of an image~$I$ as a discrete
$\mathbb{R}^c$-valued measure supported on a finite set of pixel locations. Formally, it is given by
\[
I = \sum_{y \in \Omega'} I(y)\,\delta(\cdot - y),
\]
where $\Omega' \subseteq \Omega$ is a finite set of pixels, 
$I(y) \in \mathbb{R}^c$ denotes the color at pixel~$y,$ and $\delta(\cdot -y)$ denotes the Dirac measure concentracted at~$y.$ In this formulation,~$I$ is a discrete measure concentrated at the points
of $\Omega'.$  We denote by $\mathbbm{1}_n \in \mathbb{R}^n$ the vector with all entries equal to one, and by $\chi_V$ the indicator function of a set $V,$ defined as $\chi_V(x) = 1$ if $x \in V,$ and  $\chi_V(x) = 0$ otherwise. 

\subsection{Histograms} \label{sect:hist} Let  
$I:\Omega \to \mathbb{R}^c $ be an image and $\widehat{\Omega}\subseteq \Omega$ a region. 
Let  $\{c_1, \dots, c_k\} \subseteq \mathbb{R}^c$ be a set of predefined bin centers, and define the associated Voronoi cells (bins) by
\[
V_i = \left\{ z \in \mathbb{R}^c \,\middle|\, \|z - c_i\| \leq \|z - c_j\| \text{ for all } j \neq i \right\}, \quad i = 1, \dots, k.
\] The (normalized) histogram $v_{\widehat{\Omega}}=(\nu_1,\dots,\nu_k)^T$ of a region $\widehat{\Omega} \subseteq \Omega,$
with respect to the image~$I$ and the bin centers
$\{c_1,\dots,c_k\}$, is then defined by
\[
\nu_i = \frac{\int_{\widehat{\Omega}} \chi_{V_i}(I(x)) \, dx}{\int_{\widehat{\Omega}} dx},
\]
where $\chi_{V_i} : \mathbb{R}^c \to \{0,1\}$ denotes the indicator function of the Voronoi cell $V_i.$ The histogram~$v_{\widehat{\Omega}}$ can be viewed as the discrete measure
$\sum_{i=1}^k \nu_i\,\delta(\cdot - c_i)$ supported on the bin centers.

\subsection{Optimal transport}\label{sect:OT}

Let $v\in\mathbb{R}_{\geq0}^a,$ $w\in\mathbb{R}_{\geq0}^b$ satisfy the balance equation
$v^T\mathbbm{1}_a = w^T\mathbbm{1}_b.$
The \emph{transportation polytope associated with $v$ and $w$} is defined by
\[
\mathcal{P}(v,w)
:= \left\{
P = (p_{ij})_{i,j} \in \mathbb{R}_{\ge 0}^{a \times b}
\;\middle|\;
P \mathbbm{1}_b = v,\;
P^\top \mathbbm{1}_a = w
\right\}.
\]
Transportation polytopes are well-studied objects in polyhedral geometry; we refer to~\cite{de2013combinatorics} for a survey of their combinatorial and geometric properties.

Transportation polytopes appear in a variety of computational and optimization settings. 
From the OT perspective, transportation polytopes can be used to model mass-preserving couplings between measures. For further background on OT, we refer to the monographs~\cite{villani2009optimal, peyre2019computational, friesecke2024optimal}.

In OT, the \emph{(balanced) Kantorovich problem} consists in solving  \[
\min_{P \in \mathcal{P}(v,w)}
\sum_{i=1}^a \sum_{j=1}^b d_{ij}\, p_{ij},
\]
where $(d_{ij})_{i,j}$ is a prescribed cost matrix. This is a linear program, whose solution represents an optimal transport plan between supplies~$v$ and demands~$w$, where $p_{ij}^\ast$ specifies the amount transported from index~$i$ to index~$j$. 

Two special cases are worth mentioning. First, if $v=\mathbbm{1}_a$, the feasible
set $\mathcal{P}(\mathbbm{1}_a,w)$ corresponds to a \emph{(fractional) assignment
problem} in which each supply node is assigned exactly once, while demand nodes
admit multiple assignments according to their capacities $w_j$.
Second, if $v^T\mathbbm{1}_a = w^T\mathbbm{1}_b = 1$, the vectors $v$ and $w$ can be
interpreted as discrete probability measures (histograms) supported on bin
centers $\{c_i\}_{i=1}^a \subseteq \mathbb{R}^c$ and
$\{c'_j\}_{j=1}^b \subseteq \mathbb{R}^c$, respectively. 
For the quadratic cost
$d_{ij} := \|c_i - c'_j\|^2,$ 
we define
\begin{equation}\label{eq:dis_in_H_define}
\operatorname{d}(v,w)
:= \left(
\min_{P \in \mathcal{P}(v,w)}
\sum_{i=1}^a \sum_{j=1}^b \|c_i - c'_j\|^2\, p_{ij}
\right)^{1/2}.
\end{equation}
The quantity $\operatorname{d}(v,w)$ is the \emph{$2$-Wasserstein distance
between $v$ and $w$} and defines a metric on the space of discrete probability
measures. In the following, we primarily work with the squared distance
$\operatorname{d}^2(v,w)$, corresponding to the quadratic OT cost, which is
more convenient in our setting. With a slight abuse of notation, we also write
$\operatorname{d}(\Omega_i,\Omega_j)$ for the $2$-Wasserstein distance between the 
histograms associated with the regions $\Omega_i$ and $\Omega_j$ of $I$.

\section{Related work and segmentation models}\label{sec:Relevant_Models}
Building on this notion of region comparison, we now turn to segmentation models that aim at robustness to noise and intensity inhomogeneities.

\subsection{Variational methods}

A variety of variational models have been proposed in the literature to specifically address intensity inhomogeneity and high noise levels; however, to our knowledge, none of them are superpixel-based.  Among them, the Ali-Rada (AR)~\cite{ali2018image} and Zhang-Zhang (ZZ)~\cite{zhang2015level} models can be viewed as state-of-the-art approaches. Both can be seen as advancements of the classical Local Binary Fitting (LBF) model~\cite{li2007implicit}, which we will first review. To facilitate the explanation, we will present the models for the two-phase case, where the primary goal is to segment the image into foreground and background. We denote the Heaviside function by 
$H:\mathbb{R}\to\mathbb{R},$ defined by 
$H(x)=0$ for $x<0$ and 
$H(x)=1$ for 
$x>0,$ and we denote  by $\phi:\Omega\to\mathbb{R}$ the level set function used to partition the image into foreground $\{x\in\Omega:\phi(x)>0\},$  background $\{x\in\Omega:\phi(x)<0\}$, and contour $\{x\in\Omega:\phi(x)=0\}$.

\subsubsection{Local binary fitting (LBF) model}
The LBF model~\cite{li2007implicit}, replacing the constant image intensity means of the classical Chan-Vese model~\cite{chan2001active} by adaptive functions $f_1, f_2:\Omega\to\mathbb{R}^c$ optimizes
\begin{equation*}
     \begin{aligned}
         \argmin_{\phi,\,f_1,\,f_2}\:&\lambda_1\int_{\Omega} K(x-y){\|I(y)-f_1(x)\|}^2H(\phi(y))\,\D{x}\D{y}\\
         +\,&\lambda_2\int_{\Omega} K(x-y){\|I(y)-f_2(x)\|}^2(1-H(\phi(y)))\,\D{x}\D{y}\\
         +\,&\mu\int_{\Omega} \delta(\phi(x)) \|\nabla \phi(x)\|\,\D{x}\\
         +\,&\nu\int_{\Omega} \frac{1}{2}{(\|\nabla \phi(x)\|-1)}^2\D{x},
     \end{aligned}
 \end{equation*}
 where $\lambda_1, \lambda_2, \mu, \nu>0$ are fixed parameters and $K$ is the Gaussian kernel  $K(x-y)=\exp(-||x-y||/2\sigma^2)$ with standard deviation~$\sigma$. The first two terms of the model can be viewed as region terms, aiming at fitting image intensities. The $\mu$-term aims at minimizing contour length, while the $\nu$-term resembles a signed distance function. 
 
Optimal functions $f_1,$ $f_2,$ for given $\phi,$  can be derived from  variational calculus:
\[
\begin{aligned}
f_1(x)&=
\frac{\int_\Omega K(x-y)H(\phi(y))I(y)\,dy}
     {\int_\Omega K(x-y)H(\phi(y))\,dy},\\
f_2(x)&=
\frac{\int_\Omega K(x-y)(1-H(\phi(y)))I(y)\,dy}
     {\int_\Omega K(x-y)(1-H(\phi(y)))\,dy}.
\end{aligned}
\]
Note that, in the two-dimensional case, the region terms involve four-dimensional integrals. It is also known that the LBF model can be sensitive to initialization~\cite{zhang2010active}.

\subsubsection{Ali-Rada (AR) model} Given the two spatially varying average functions $f_1, f_2$ in the LBF model, a statistically locally computed approximation $\bar{I}$ of the image $I$ can be defined via 
\begin{equation*}
    \begin{aligned}
        \bar{I}(x)=f_1(x)H(\phi(x))+f_2(x)(1-H(\phi(x))).
    \end{aligned}
\end{equation*}
The AR model~\cite{ali2018image}, utilizing a
combination of this locally computed denoising constrained surface
and a denoising fidelity term, optimizes

\begin{equation*}
     \begin{aligned}
         \argmin_{\phi,\, f_1,\, f_2}\:&\lambda\int_{\Omega} \|I(x)-\bar{I}(x)\|^2\,\D{x}\\
         +\,&(1-\lambda)\int_{\Omega}\sum_{k=1}^c \Big(\log(\bar{I_k}(x))+I_k(x)/\bar{I}_k(x)\Big)\,\D{x}\\
         +\,&\mu\int_{\Omega} \delta(\phi(x)) \|\nabla \phi(x)\|\,\D{x},
     \end{aligned}
 \end{equation*}
where $\mu>0$ and $\lambda\in [0,1]$ are given constants.

 \subsubsection{Zhang-Zhang (ZZ) Model}
The ZZ model \cite{zhang2015level} also employs local information to mitigate the effect of  inhomogeneities. Introducing an unknown bias field $B:\Omega\to \mathbb{R},$ it is assumed that the intensities in a neighborhood of a pixel follow a Gaussian distribution, yielding the  model
\begin{equation*}
     \begin{aligned}
         \argmin_{\phi,\,B,\, c_1,\, \sigma_1, \,c_2, \,\sigma_2}&\int_{\Omega} K_\rho(x-y)\Big(\log(\sigma_1)+\frac{{\|I(y)-B(x)c_1\|}^2}{2\sigma_1^2}\Big)H(y)\,\D{x}\D{y}\\
         +\,&\int_{\Omega} K_\rho(x-y)\Big(\log(\sigma_2)+\frac{{\|I(y)-B(x)c_2\|}^2}{2\sigma_2^2}\Big)
         (1-H(y))\,\D{x}\D{y}\\
         +\,&\mu\int_{\Omega} \delta(\phi(x)) \|\nabla \phi(x)\|\,\D{x},
     \end{aligned}
 \end{equation*}
where $\mu>0$ is a given constant and $K_\rho$ denotes the characteristic function of the radius-$\rho$ ball. The rationale for incorporating $c_1,c_2\in\mathbb{R}^c$ and $\sigma_1,\sigma_2\in\mathbb{R}$ lies in modeling the intensities around the foreground pixel \( x \) as a Gaussian distribution with mean~\( B(x)c_1 \) and standard deviation \( \sigma_1 \), while for background pixels \( x \), the intensities are modeled by a Gaussian distribution with mean \( B(x)c_2 \) and standard deviation \( \sigma_2 \).
 

\subsection{Deep learning models: The Segment Anything Model (SAM)} 
Another class of models is given by deep learning models. The past few years have witnessed a surge in deep learning models for image segmentation. Among these models, SAM, introduced in~2023~\cite{kirillov2023segment}, stands out as a state-of-the-art method that has been pre-trained on an extensive dataset of over  1 billion masks across 11 million images. At its core, SAM utilizes a vision transformer architecture~\cite{dosovitskiy2020image}. In this work, we employ SAM’s zero-shot capability to perform object segmentation without requiring prior training on the objects in the image. We apply SAM in both prompted and non-prompted modes, where the prompts function as marker points to assist the model in identifying and segmenting the objects of interest. Recent performance evaluations of the SAM model can be found in~\cite{ALI2025102473, ji2024segment, 10548751}.

\subsection{Superpixel-based segmentation models}
Superpixel algorithms group pixels with similar color
and other low-level properties. They are increasingly incorporated into image segmentation algorithms, see, e.g.,~\cite{hao2016superpixel, zhou2017superpixel, 6466882, zhou2015learning,li2019fast,xia2019ivus,liu2022superpixel,liu2012superpixel,wang2017optimal,doi:10.1080/01431161.2017.1354266,xiang2020fast}. Among these approaches, the SMST model from~\cite{wang2017optimal} is the most methodologically similar to our proposed model. Consequently, we provide a brief overview of SMST and include it in our comparison. We begin with a summary of the superpixel generation methods SLIC and Power-SLIC, underlying SMST and our approach, respectively.

\subsubsection{Superpixel generation: SLIC}\label{sec:slic}
Among the methods available for superpixel generation (see~\cite{stutz2018superpixels,wang2017superpixel} for a comprehensive evaluation), SLIC~\cite{achanta2012slic}, an adaptation of the $k$-means clustering method~\cite{lloyd1982least,macqueen1967some} for superpixel generation,  stands out as one of the most popular and widely used techniques. Given a discretized image of~$N$ pixels, SLIC clusters pixels in the joint spatial–feature space $\mathbb{R}^d\times\mathbb{R}^c,$ where each pixel~$p$ is represented by $(x_p,I(x_p)),$ with $I(x_p)$ taking values in CIELAB space~\cite{robertson1977cie} for color images and in~$\mathbb{R}$ for grayscale images.

The algorithm comprises two main phases. In the first phase, SLIC depends on two parameters,
$m$ and $\alpha$, where $m$ specifies the target number of superpixels and~$\alpha$ is a compactness parameter ($\alpha=10$ is common in 2D). The cluster centers are initialized on a regular grid with spacing
$h = \sqrt{N/m}.$ During the assignment step, each pixel is compared only to cluster centers within a spatial
neighborhood of size proportional to~$h.$  The similarity between a pixel~$p$ and a cluster center $s$ is measured by the
weighted squared Euclidean distance
\[
\mathrm{dist}^2(p,s)
= \|x_p - x_s\|^2
+ \frac{h^2}{\alpha^2}\,\|I(x_p) - I(x_s)\|^2.
\] As in $k$-means, the centers are updated as the centroids of the updated clusters, and the process is iterated until some stopping criterion is fulfilled. In a second phase, connectivity is enforced by relabeling small disconnected components to adjacent superpixels.

\subsubsection{Superpixel generation: Power-SLIC}\label{sec:powerslic}

In our proposed method, we will employ a recent superpixel generation method called Power-SLIC~\cite{fiedler2020power}. This method improves over SLIC in terms of speed, noise robustness, and compactness. Conceptually, Power-SLIC computes, heuristically, a solution of the fractional assignment problem discussed in Section~\ref{sect:OT}.

More specifically, Power-SLIC first computes pixel subsets $C_j,$ $j=1,\dots,m,$ via the first phase of SLIC. Then a principal component analysis is performed, i.e., for each~$j,$ the sample covariance matrix~$\Sigma_j,$ the centroid $c(C_j),$ and the cardinality $|C_j|$ of the data set $C_j$ are calculated. 
 Setting $A_j:=\Sigma_j^{-1},$ this defines an  ellipsoidal norm $||\cdot||_{A_j}$ via $||x||_{A_j}:=\sqrt{x^TA_jx},$ If one then solves the fractional assignment problem from Section~\ref{sect:OT} with $a=N,$ $b=m,$ $w_j=|C_j|,$  and quadratic costs $d_{ij}=||x_i-c(C_j)||^2_{A_j},$ for $i=1,\dots,N$ and $j=1,\dots,m,$ then its dual linear problem yields parameters $\mu_1^*,\dots,\mu_m^*,$ which define cells

\[V_i:=\{x\in\Omega: ||x-c(C_i)||^2_{A_i}-\mu^*_i\leq ||x-c(C_j)||^2_{A_j}-\mu^*_j, \quad \textnormal{for all } j=1,\dots,m\};\] see~\cite{fiedler2020power} for details. 
 The  $m$-tuple $(V_1,\dots,V_m)$ of cells  is referred to in the literature as \emph{anisotropic power diagram} or \emph{generalized balanced power diagram}~\cite{altendorf20143d,alpers2015generalized}. These  possess an inherent geometric structure, generalizing the concepts of \emph{Voronoi} and \emph{Laguerre diagrams} (see, e.g.,~\cite{aurenhammer2013voronoi}). The discrete superpixels are then the sets $V_i\cap\Omega',$ $i=1,\dots,m.$ To reduce computational complexity, Power-SLIC uses a  heuristic parameter update and
sets
\[
\mu_i^\ast := \left(\frac{|C_i|}{\kappa_d}\sqrt{\det(A_i)}\right)^{2/d},
\]
with \(\kappa_d\) denoting the volume of the \(d\)-dimensional unit ball; see~\cite{alpers2023dynamic} for a discussion of the underlying rationale in the context of discriminant analysis.

\subsubsection{Superpixel-based minimum spanning tree (SMST) model} 
Superpixel-based segmentation methods frequently employ a combination of merging and splitting phases. The SMST model described in~\cite{wang2017optimal} constitutes a state-of-the-art example of this paradigm.

The SMST model performs the segmentation in four stages. First, superpixels are generated via~SLIC. The superpixels are represented as nodes in a graph, with edges weighted by the dissimilarity between nodes, and a minimum spanning tree is then computed based on this graph. Second, superpixels are gradually merged using the MST structure (the implementation uses the approach from~\cite{guo2008regionalization}). During the merging several key parameters, such as the inter-cluster variance, are calculated and monitored. In a third stage,  the parameters are analyzed to yield an estimate of the  number of objects in the image. In the fourth stage, the superpixels are split to generate the initial MST, and re-merging is carried out until the determined number of objects is reached.

\section{Proposed SP (superpixel) model}\label{sec:segmodels}

We propose an adjacency-constrained,
optimal-transport-based region-merging model,
referred to as the \emph{SP (superpixel) model},
formulated on the region-adjacency graph of an
initial superpixel partition. Let
$\mathcal{Q}_m = \{\Omega_1,\dots,\Omega_m\}$ denote an oversegmentation of the discretized image~$I$ into~$m$ nonempty superpixels, and let~$v_{\Omega_i}$ denote the histogram associated with region $\Omega_i$; see Section~\ref{sect:hist}. We consider the region-adjacency graph
$G = (V,E),$ $V = \{1,\dots,m\},$ 
where $(i,j)\in E$ if and only if $\Omega_i$ and~$\Omega_j$ share a common boundary. For each adjacent pair $(i,j)\in E$, we define the dissimilarity
$
E_{ij} = \operatorname{d}^2(v_{\Omega_i},v_{\Omega_j}),$ 
the squared $2$-Wasserstein distance between their color histograms; see~\eqref{eq:dis_in_H_define}.

To each region $\Omega_i$ we associate a heterogeneity
value $h_i$, initialized as $h_i=0$.
For each adjacent pair of regions $\Omega_i$ and
$\Omega_j$, we define the \emph{regularized merge cost}
\begin{equation}\label{eg:mergecost}
\kappa_{ij}=E_{ij}-h_i-h_j.
\end{equation}
The \emph{associated greedy algorithm} iteratively merges
an adjacent pair of regions with minimum regularized
merge cost. 
Whenever two regions $\Omega_i$ and $\Omega_j$ are merged,
we update $h_i\gets E_{ij}$, retaining index~$i$.
Index $j$ is then retired, and the labels and
heterogeneity values of all other regions are unchanged.
The rationale behind this construction is to encourage
merges that are both locally coherent in color distribution
and consistent with the hierarchical structure of previously
merged regions. The heterogeneity values retain information
about previous merges, so that the regularized cost measures
the change in heterogeneity associated with merging two regions.

A segmentation with $n<m$ regions is described by a sequence of admissible merges $\mathcal{M} = \bigl((i_1,j_1),\dots,(i_{m-n},j_{m-n})\bigr),$ 
where each $(i_k,j_k)$ is an adjacent pair in the current region-adjacency graph. For such a merge sequence $\mathcal{M}$, we define the \emph{regularized pairwise optimal-transport energy}
\[
\mathcal{E}_{\mathrm{reg}}(\mathcal{M})
=
\sum_{k=1}^{m-n}
\Bigl( \operatorname{d}^2(v_{\Omega_{i_k}},v_{\Omega_{j_k}}) - h_{i_k} - h_{j_k} \Bigr),
\]
with the heterogeneity values $h_i$ evolving according to the update rule above.  The \emph{global SP optimization problem} associated with
the SP model is to minimize
\[
\mathcal E_{\mathrm{reg}}(\mathcal M)
\quad\text{over }\mathcal M\in\mathfrak A_{m\to n},
\]
where $\mathfrak{A}_{m\to n}$ denotes the set of all admissible merge sequences reducing the initial partition~$\mathcal{Q}_m$ to a partition of size $n$. Unlike the greedy  algorithm, this optimization considers all admissible merge sequences.

The global SP optimization problem is
computationally intractable in general.
In Appendix~\ref{app:complexity}, we establish
$\mathsf{NP}$-completeness of the associated decision
problem, even for two final regions, two-bin
histograms, and region-adjacency graphs of
treewidth at most three, with the graph and
histograms supplied as input. This motivates
the use of greedy optimization strategies,
which we develop in the following section.

\section{Greedy algorithms for the SP model}\label{sect:SPAlgs} We introduce three variants of the greedy
 algorithm for the SP model,
corresponding to different assumptions on
the available input information.

\subsection{Unsupervised SP algorithm} The unsupervised SP algorithm takes as input an image
$I$, a prescribed number of initial superpixels
$m\in\mathbb{N}$, a target number of regions
$n\leq m$, and a number $k\in\mathbb{N}$ of color
representatives used in the histogram computations.
It constructs an initial partition of $I$ into
$m$ superpixels using Power-SLIC and computes
their color histograms and region-adjacency graph.

Starting from this representation, the algorithm
iteratively merges adjacent regions with minimum
regularized merge cost~\eqref{eg:mergecost} until
exactly $n$ regions remain, updating the region
histograms and adjacency graph after each merge.
The resulting sequence of merges defines a
hierarchy of progressively coarser partitions.
The full procedure is summarized in
Algorithm~\ref{alg:algorithm}.

\begin{algorithm}[h]
\caption{Unsupervised SP optimization algorithm}\label{alg:algorithm}
\begin{algorithmic}[1]
\Require Discretized image $I$, number of superpixels $m\in\mathbb{N}$, number $n\leq m$ of objects, number of color representatives $k\in\mathbb{N}$.
\Ensure Partition $\{\Omega_i:\Omega_i\neq\emptyset\}$ of $I$ into $n$ superpixel regions.
\State Compute $k$ representative colors of $I$ for histogram-based distance computations.
\State Partition $I$ into $m$ nonempty superpixels $\Omega_1,\,\Omega_2,\,\dotsc\,\Omega_m$ using Power-SLIC.  \label{alg:linepowerslic}
\State Initialize region heterogeneity levels $h_i \gets 0$ for all $i=1,\dots,m$.
\State Initialize $E_{ij} \gets \operatorname{d}^2(\Omega_i,\Omega_j)$ for all adjacent region pairs $i<j$. \label{alg:linedist1}
\State Initialize a priority queue $Q$ with entries
\[\bigl(E_{ij} - h_i - h_j,\; i,\; j\bigr) \]
for all adjacent region pairs $i<j$.
\State Set $r\gets m$.

\While{$r>n$}
    \State Extract $(\kappa,i,j)\in Q$ with minimal key $\kappa.$
    \If{($\Omega_i \neq \emptyset$ and $\Omega_j \neq \emptyset$)}
        \State Merge $\Omega_j$ into $\Omega_i$. \label{alg:merge_step}
        \State  $h_i \gets E_{ij}$.
        \State $r \gets r-1$.
        \State  $\mathrm{LT}(r) \gets h_i$. \Comment{Used only when $n$ is determined via ROC-LT.}
        \For{each region $\Omega_k$ adjacent to $\Omega_i$}
            \State Recompute $E_{ik} \gets \operatorname{d}^2(\Omega_i,\Omega_k)$.\label{alg:linedist2}
            \State Insert $(E_{ik}-h_i-h_k,\; i,\; k)$ into $Q$.
        \EndFor
    \EndIf
\EndWhile
\end{algorithmic}
\end{algorithm}

Although Algorithm~\ref{alg:algorithm} does not generally
solve the global SP optimization problem, it admits
theoretical guarantees under suitable separation
assumptions. In Appendix~\ref{app:recovery}, we
show that the algorithm exactly recovers a prescribed
two-region segmentation when the region histograms
are sufficiently separated. We also establish an
additive approximation bound for the regularized
energy, quantifying the difference between the
greedy solution and the global optimum.

\subsection{SP algorithm with automatic region number selection} Algorithm~\ref{alg:algorithm} can be extended to automatically determine the number of regions from the evolution of the merging process. The algorithm is run until a single region remains, yielding a sequence of merges. At each iteration, two regions~$\Omega_i$ and~$\Omega_j$ minimizing the regularized merge cost $\kappa_{ij}$~(see~\eqref{eg:mergecost}) are merged, and we define
\begin{equation*}
\mathrm{LT}(r) := \operatorname{d}^2(\Omega_i,\Omega_j),
\end{equation*}
where~$r$ denotes the number of regions after merging. The resulting sequence \(\{\mathrm{LT}(r)\}\) records the costs of successive merges as~$r$ decreases. Following the rate-of-change (ROC) criteria introduced and motivated in \cite{wang2017optimal,drǎguct2010esp}, we consider the relative discrete variation
$\mathrm{ROC}(r)
=
(\mathrm{LT}(r-1)-\mathrm{LT}(r))/\mathrm{LT}(r),$
 whose  local maximizers~$r^*$ indicate suitable values of~$n$ (selected by visual inspection in all experiments in this paper). The corresponding partition is obtained by reversing the final~$n-1$ merges.

\subsection{Marker-based (supervised) SP algorithm}\label{SPmarker}
When user-specified markers 
$(p_i,I_i)$ are available, Algorithm~\ref{alg:algorithm} replaces the prescribed number of regions by marker locations $p_i$ with associated classes~$I_i,$ assuming exactly~$n$ distinct classes.

After superpixel initialization (line~\ref{alg:linepowerslic}), marker regions~$M=\{M_I\}_I$ are defined as unions of initial superpixels containing markers of
class~$I$.

In lines~\ref{alg:linedist1} and~\ref{alg:linedist2}, the dissimilarity $\operatorname{d}^2(\Omega_i,\Omega_j)$ is replaced by a
 marker-aware dissimilarity~$\operatorname{d}_M^2(\Omega_i,\Omega_j)$, which defines region dissimilarity relative to a class-specific reference region in the presence of a unique marker class.  It is given by
\[
\operatorname{d}_M^2(\Omega_i,\Omega_j)
:= \operatorname{d}^2(\Omega_i,C(\Omega_i,\Omega_j))+\operatorname{d}^2(\Omega_j,C(\Omega_i,\Omega_j)),
\]
where
\[
C(\Omega_i,\Omega_j):=
\begin{cases}
M_I, & \text{if } \Omega_i\cup \Omega_j \text{ contains markers of exactly one class } I,\\
\Omega_i, & \text{otherwise.}
\end{cases}
\] 

Moreover, the loop condition is amended to $r>n$ and $Q\neq\emptyset$.
Finally, the merge condition $\Omega_i\neq\emptyset$ and $\Omega_j\neq\emptyset$ is strengthened by
requiring that all markers contained in $\Omega_i\cup\Omega_j$ belong to the same class. Note that the output is a partition of~$I$ into superpixel regions consistent with the prescribed marker classes; however, the number of regions is not fixed a priori and may exceed the number of classes.

\subsection{Computational cost}

We analyze Algorithm~\ref{alg:algorithm} in the unit-cost RAM model (see, e.g.,~\cite{cormen2009introduction,papadimitriou1994computational}), assuming standard
implementations of priority queues (binary heaps) and adjacency lists, with
region merges performed using standard representative-based bookkeeping.  The superpixel partition induces a region adjacency graph with $E$ edges and $m$ vertices. 
We further assume that each pixel color can be assigned to one of the $k$
representative colors in constant time, for instance via a precomputed lookup
table built during the representative-color computation. Let $T_{\mathrm{rep}}(N,k)$ denote the worst-case running time required to compute
the~$k$ representative colors in our image~$I$ containing~$N$ pixels, and let $T_{\mathrm{OT}}(k)$ denote the
worst-case running time required to solve a balanced transportation problem of
size at most~$k$.

\begin{proposition}\label{prop:computational-cost} Algorithm~\ref{alg:algorithm} runs in worst-case time \[
O\bigl(
T_{\mathrm{rep}}(N,k) + N + m^2 (\log m + T_{\mathrm{OT}}(k)) + mk
\bigr).
\]
\end{proposition}

The proof is given in Appendix~\ref{app:coputational-cost}.

Balanced transportation problems with $k$ supply and $k$ demand nodes can be
solved in strongly polynomial time in the unit-cost RAM model; see, for example,
the cost-scaling algorithm~\cite{orlin1993}. In our implementation, the
distance computations~$\operatorname{d}(\Omega_i,\Omega_j)$ are carried out using
the network simplex--based optimal transport solver of Bonneel et~al.\
\cite{bonneel2011displacement,BonneelCode}. While this solver does not admit a
known polynomial worst-case complexity bound, it is empirically observed to
exhibit near-quadratic running time in many practical instances~\cite{bonneel2011displacement}. In all
experiments reported below, we fix $k=15$.

Concerning the term $T_{\mathrm{rep}}(N,k)$, our implementation computes
representative colors from an auxiliary Power-SLIC partition into a fixed number
of regions (300 in all experiments). Color vectors are averaged within each
auxiliary region, and the $k$ representative colors are selected as the $k$ most
frequent averages. 
Since the number of auxiliary regions is fixed,
the representative-color computation takes
$T_{\mathrm{rep}}(N,k)=O(N)$ time. For fixed $k$,
Proposition~\ref{prop:computational-cost}
therefore yields a worst-case bound of
$O(N+m^2\log m)$ when the transportation
problems are solved using a strongly
polynomial algorithm.

\section{Numerical experiments}\label{sec:experiments}
In this section, we evaluate the quality and efficiency of our model under intensity inhomogeneities. For comparison, we include the following models:

\begin{itemize}
    \item \textbf{SP}: our SP model implemented
using the greedy algorithms
described in Section~\ref{sect:SPAlgs}, 
    \item \textbf{AR}: the Ali-Rada model \cite{ali2018image}, 
    \item \textbf{ZZ}: the Zhang-Zhang model \cite{zhang2015level},
    \item \textbf{SMST}: the superpixel-based minimum spanning tree model \cite{wang2017optimal}, and
    \item \textbf{SAM}: the Segment Anything model \cite{kirillov2023segment}.
\end{itemize}

We next describe the implementation details and computational settings for the models under comparison.

SP is implemented in C++, using the three algorithmic variants described in Section~\ref{sect:SPAlgs},  incorporating the network simplex solver~\cite{BonneelCode} for OT computations. 
The MATLAB code is (obtained by compiling the C++ code into MEX modules) is available at \href{https://github.com/HJS-design/Superpixel-Based-Image-Segmentation-Using-Squared-2-Wasserstein-Distances.git}{GitHub}.
 As remarked before, we use ~$k=15$ color representatives in all experiments. 

For the AR and ZZ models, we employ the level-set framework of \cite{roberts2019convex}, which uses additive operator splitting (AOS) with convex relaxation and avoids explicit reinitialization. AOS decomposes each 2-D Laplacian update into two 1-D tridiagonal solves. The original implementation relies on MATLAB’s generic solver \texttt{mldivide}, which does not exploit this structure, so we replace it with the linear-time Thomas algorithm \cite{higham2002accuracy}. In the ZZ model, we reformulate the quadruple integral as a set of 2D convolutions, reducing the cost from 
quadratic to  quasi-linear using FFT acceleration.

The SMST model is implemented in C++ using a minimum-spanning-tree representation together with a custom implementation of the REDCAP procedure \cite{guo2008regionalization} as no public implementation is available. Subtree variances are expressed via $\mathrm{var}(X)=\operatorname {E} \left[X^{2}\right]-\operatorname {E} [X]^{2}$, so that all required statistics reduce to sums of $X,$ $X^2,$ and region areas over subtree nodes. A depth-first traversal linearizes the tree, mapping each subtree to a contiguous memory interval and enabling all aggregates to be computed by prefix sums; the interval structure is updated after each bipartition.

 The SAM model was executed on Google Colab using a Tesla T4 GPU. All other models were run on a Lenovo laptop equipped with a 13th-generation Intel(R) Core(TM) i9-13900HX CPU and 40 GB RAM.

\subsection{Datasets}\label{sect:datasets}
We evaluate three datasets exhibiting challenges typical in microscopy and biomedical imaging, including illumination variability, heterogeneous object morphology, large object counts, and regional intensity inhomogeneity.

\textbf{Dataset~1} comprises a synthetic video and two hematoxylin and eosin (H\&E)-stained images collected from public online sources~\cite{systemicMastocytosis2020,chan2014wonderful}. H\&E staining is standard in histopathology, with nuclei appearing blue–purple and cytoplasm pink. The images are provided in RGB format with varying image sizes and are used solely for qualitative illustration and timing comparisons, as no ground-truth annotations are available.
The synthetic video consists of~74 MATLAB-generated frames. The first frame shows a 3D mitochondrial mesh generated via the \texttt{isosurface} function and rendered with Phong illumination, with the light source positioned above and to the right of the camera. Subsequent frames are obtained by rotating this initial view. This dataset was constructed to study segmentation under pronounced frame-to-frame illumination changes.

\textbf{Dataset 2}, released as supplementary material in~\cite{wienert2012detection}, contains 36 digitized H\&E-stained microscopy images (600 × 600 pixels) with 7931 nuclei annotated by three expert pathologists; only consensus annotations are retained. It is designed for cell nuclei detection.

\textbf{Dataset 3} consists of 50 annotated H\&E-stained histological images from triple-negative breast cancer patients~\cite{naylor2018segmentation}. The images were acquired at 40× magnification using a Philips Ultra Fast Scanner (1.6RA) at the Curie Institute and contain a total of 4022 annotated cells.

Representative samples from the three datasets are shown in Fig.~\ref{fig:sample_images}. In (a), the first frame of the synthetic video in Dataset 1 illustrates the strong illumination gradients produced by distance-dependent shading together with orientation-dependent interactions between vertex normals and the light direction. The H\&E images in (b) and (c), also from Dataset 1, depict a mast cell infiltrate from a patient with non-alcoholic steatohepatitis and cirrhosis and, respectively, a sample from a patient with sclerosing polycystic adenosis of the parotid gland; the former consists primarily of purple nuclei, white cytoplasm, and pink extracellular tissue. Subfigures (d) and (e) show two examples from Dataset 2 accompanied by their ground-truth nuclei annotations, while (f) and (g) present a representative image from Dataset 3 together with its corresponding cell mask.

\begin{figure}[htb]
\centering
\subfloat[\label{fig_sub:rending_image}]{\includegraphics[width=0.249\linewidth]{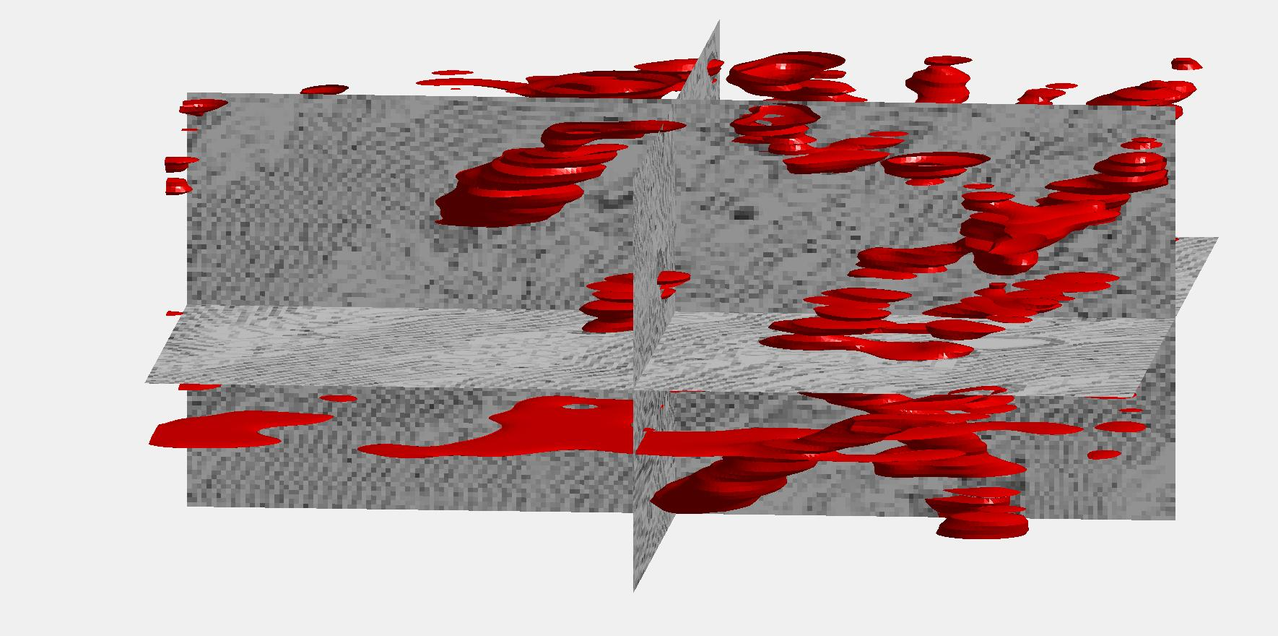}}
\hfill
\subfloat[\label{fig_sub:he1_image}]{\includegraphics[width=0.249\linewidth]{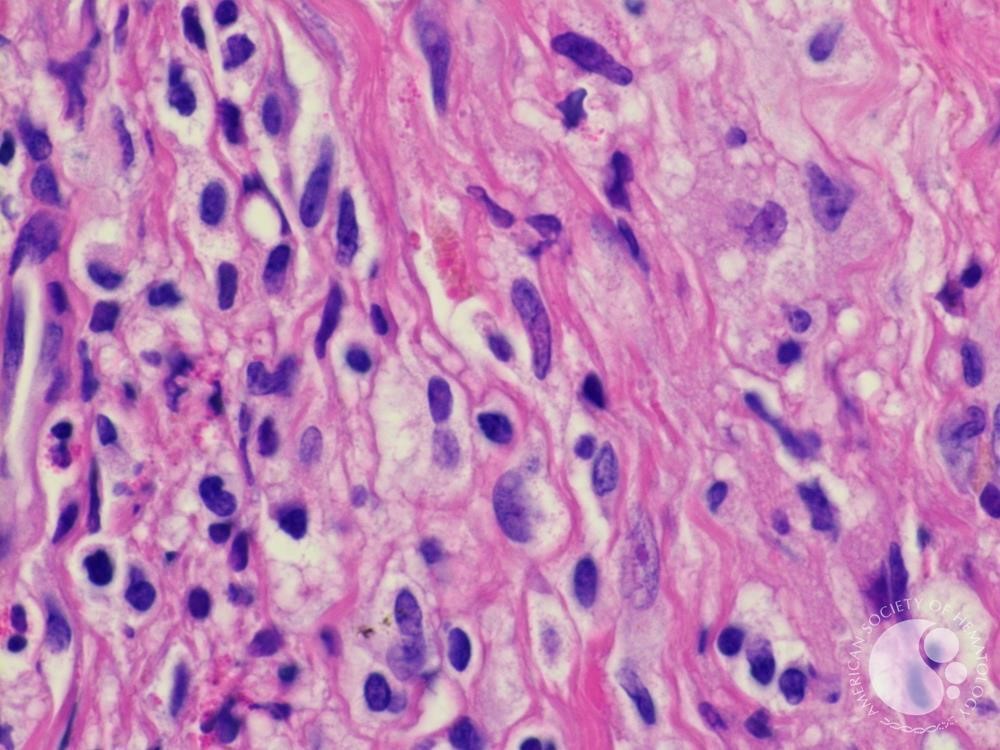}}
\hfill
\subfloat[\label{fig_sub:he2_image}]{\includegraphics[width=0.249\linewidth]{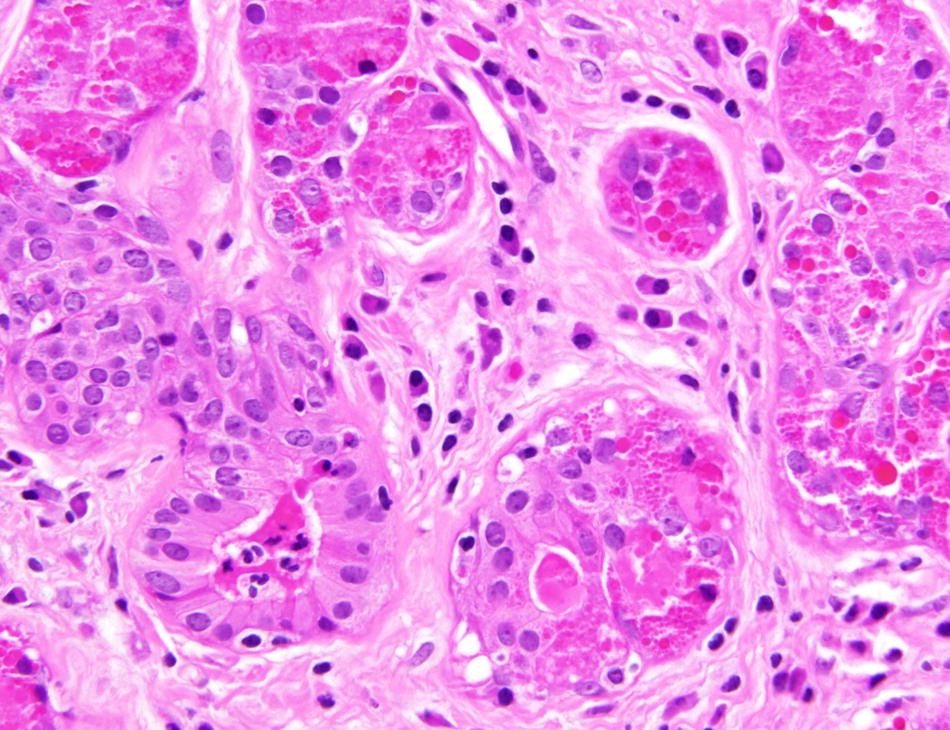}}

\subfloat[\label{fig_sub:he3_image}]{\includegraphics[width=0.249\linewidth]{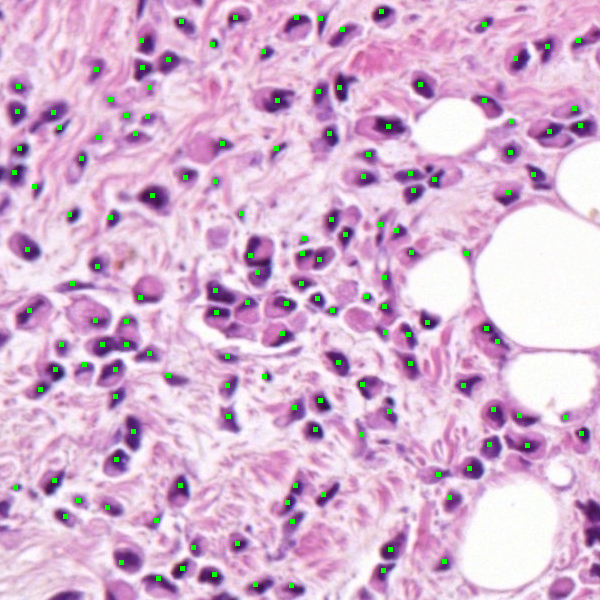}}
\hfill
\subfloat[\label{fig_sub:he4_image}]{\includegraphics[width=0.249\linewidth]{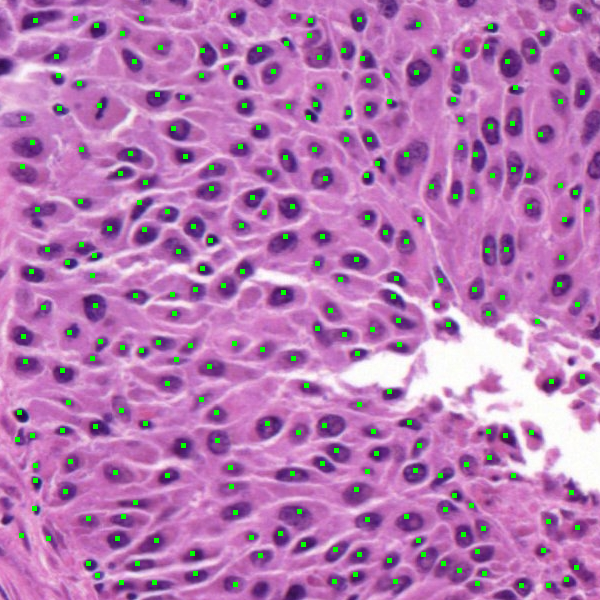}}
\hfill
\subfloat[\label{fig_sub:TNBC_image}]{\includegraphics[width=0.249\linewidth]{TNBC_segmentation/fist_image.png}}
\hfill
\subfloat[\label{fig_sub:TNBC_image_with_label}]{\includegraphics[width=0.249\linewidth]{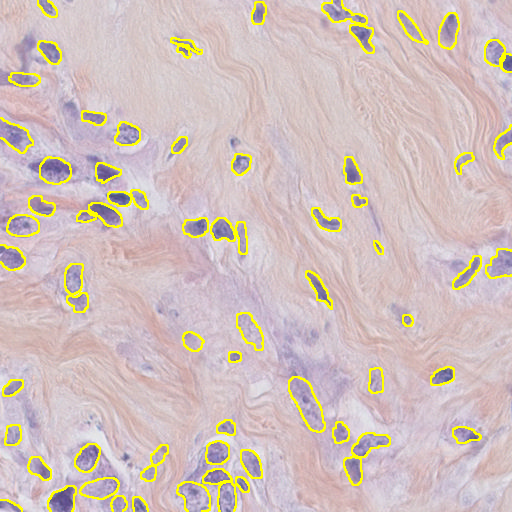}}
\caption{Samples from Datasets 1–3. (a) First frame of the 74-frame video in Dataset 1 (1770 × 880) with a superimposed red mesh. (b,c) H\&E-stained tissue images from Dataset 1 (1000 × 750; 950 × 730). (d,e) Cell images from Dataset 2 (600 × 600) with nuclei annotations. (f,g) Image from Dataset 3 (512 × 512) with annotation.}\label{fig:sample_images}
\end{figure}

\begin{figure}[htb]
\centering

\subfloat[\label{fig_sub:decide_number_curve1}]{\includegraphics[width=0.24\linewidth]{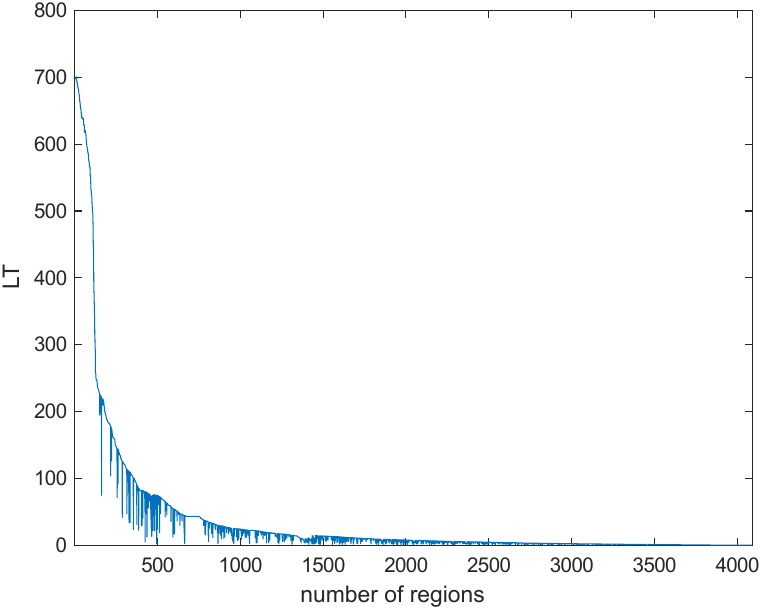}}
\hfill
\subfloat[\label{fig_sub:decide_number_curve1_diff}]{\includegraphics[width=0.24\linewidth]{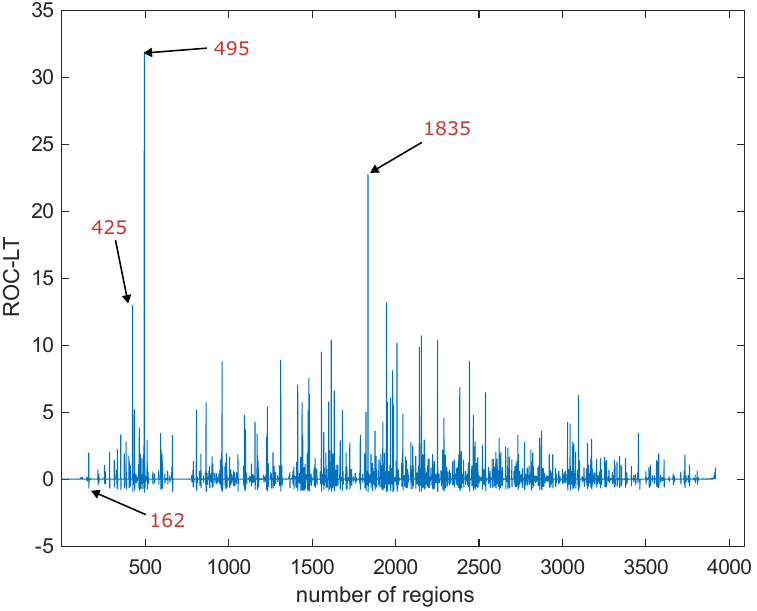}}
\hfill
\subfloat[\label{fig_sub:decide_number_curve2}]{\includegraphics[width=0.24\linewidth]{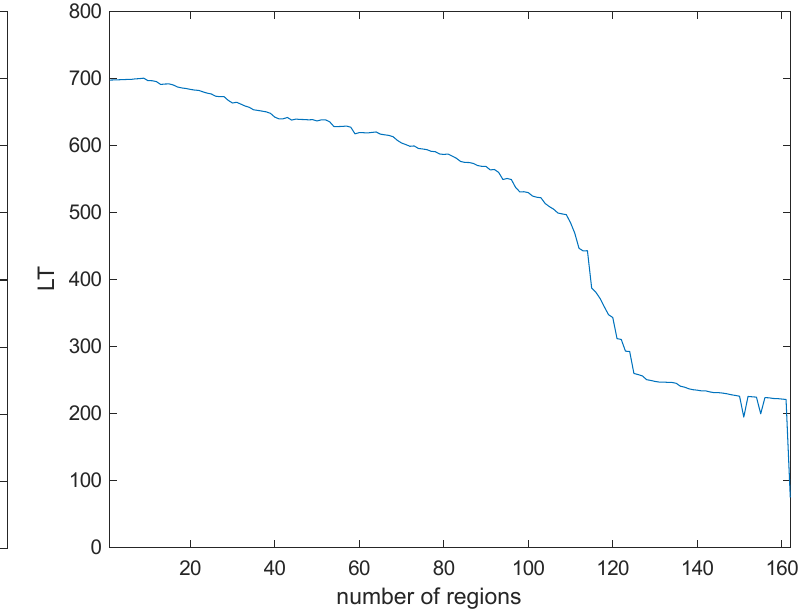}}
\hfill
\subfloat[\label{fig_sub:decide_number_curve2_diff}]{\includegraphics[width=0.24\linewidth]{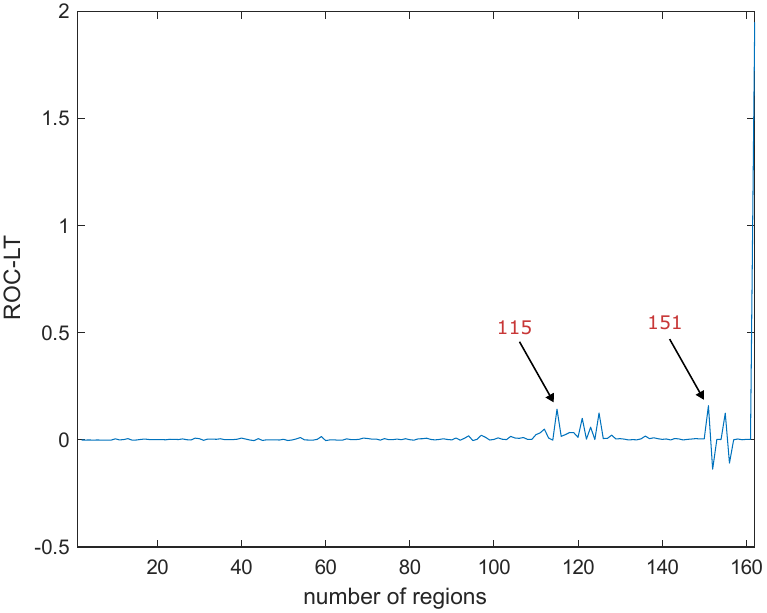}}

\subfloat[\label{fig_sub:decide_number_parameters_1835}]{\includegraphics[width=0.199\linewidth]{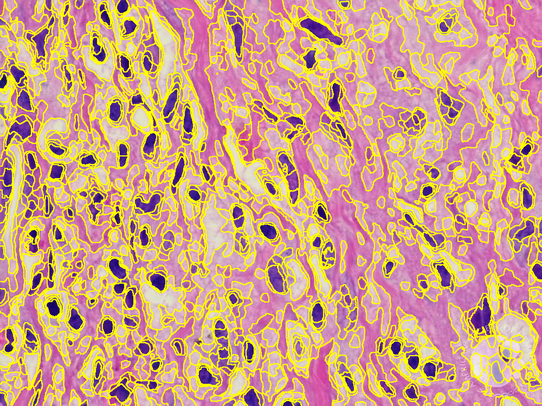}}
\hfill
\subfloat[\label{fig_sub:decide_number_parameters_495}]{\includegraphics[width=0.199\linewidth]{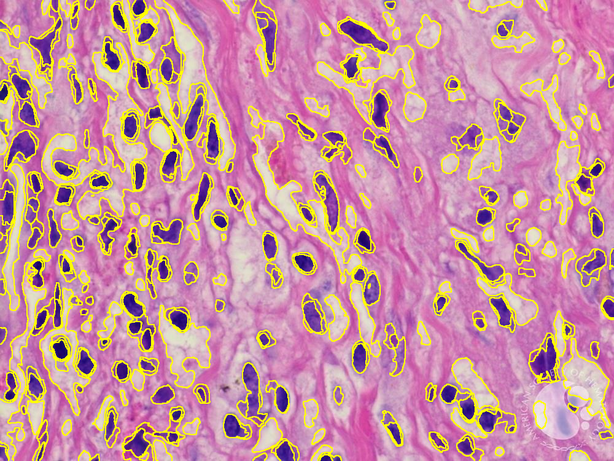}}
\hfill
\subfloat[\label{fig_sub:decide_number_parameters_425}]{\includegraphics[width=0.199\linewidth]{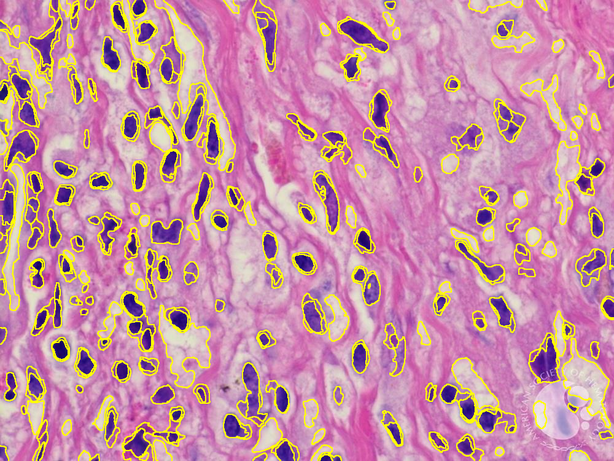}}
\hfill
\subfloat[ \label{fig_sub:decide_number_parameters_151}]{\includegraphics[width=0.199\linewidth]{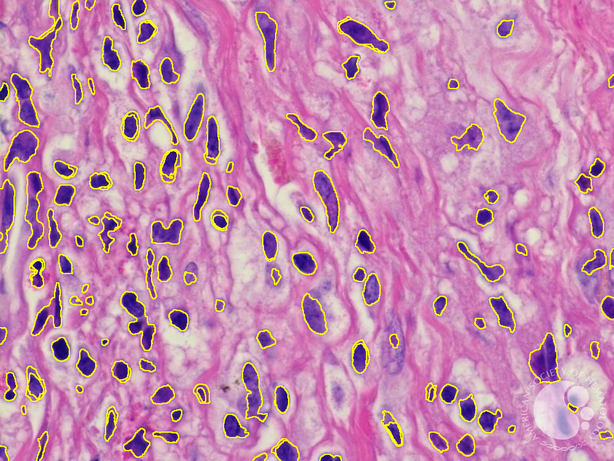}}
\hfill
\subfloat[ \label{fig_sub:decide_number_parameters_115}]{\includegraphics[width=0.199\linewidth]{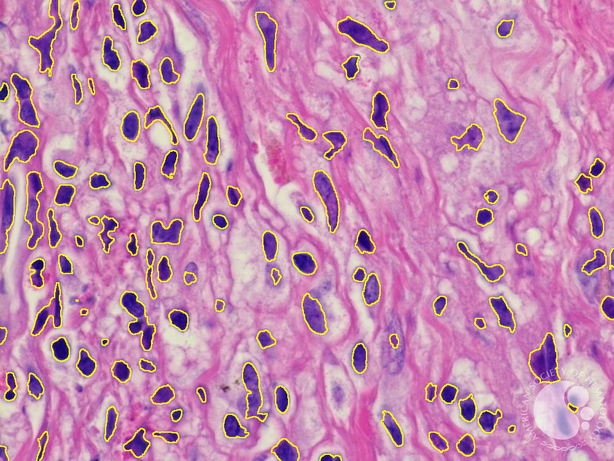}}

\caption{Selection of the optimal number of regions. The LT curve (a) shows a general decrease with minor fluctuations, while the ROC--LT curve (b) remains essentially zero over $[1,\,162]$, suggesting an optimal region number below $162$. Zoomed views (c, d) reveal peaks at $115$ and $151$. Segmentations with 1835, 495, 425, 151, and 115 regions are shown in (e--i).}\label{fig:decide_number}
\end{figure}

\begin{figure}[htb]
\centering
\subfloat[\label{fig_sub:noise_image}]{\includegraphics[width=0.15\linewidth]{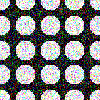}}
\hfill
\subfloat[\label{fig_sub:noise_in}]{\includegraphics[width=0.15\linewidth]{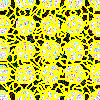}}
\hfill
\subfloat[\label{fig_sub:noise_r_200}]{\includegraphics[width=0.15\linewidth]{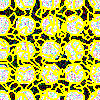}}
\hfill
\subfloat[\label{fig_sub:noise_r_100}]{\includegraphics[width=0.15\linewidth]{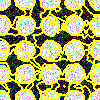}}
\hfill
\subfloat[\label{fig_sub:noise_out_sp}]{\includegraphics[width=0.15\linewidth]{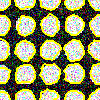}}
\hfill
\subfloat[\label{fig_sub:noise_out_super}]{\includegraphics[width=0.15\linewidth]{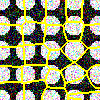}}
\caption{Segmentation vs. Oversegmentation. (a) Noisy grayscale image containing ~25 white objects. (b)  Power-SLIC result with 300 superpixels, serving as  input to our algorithm. (c–e) Intermediate and final steps of the algorithm, merging into 200, 100, and finally 25 regions, respectively. (f)  Power-SLIC output with~25 superpixels.}\label{fig:circle_noise}
\end{figure}

\section{Results}\label{sect:results}
In this section, we present both qualitative and quantitative evaluations of the proposed method in comparison with established approaches. The analysis addresses segmentation accuracy under varying conditions, computational performance, and the influence of user interaction.

\subsection{Qualitative results}
We first investigate qualitative aspects of the segmentations.

\subsubsection{Optimal number of segmentation regions} Based on the image in Fig.~\ref{fig_sub:he1_image}, we demonstrate how the ROC mechanism can identify the optimal number of segmentation regions. The LT and ROC–LT curves for this image are shown in Figs.~\ref{fig_sub:decide_number_curve1}-\ref{fig_sub:decide_number_curve1_diff}, respectively. 
The curves show only minor fluctuations for region numbers below~162 indicating that suitable segmentations are likely to lie within the interval \([1,\,162]\).  Within this range (Figs.~\ref{fig_sub:decide_number_curve2}-\ref{fig_sub:decide_number_curve2_diff}), local optima seem to lie at 115 and 151, suggesting these as plausible object counts. Segmentation results corresponding to the local maxima in the higher region interval \([162,\,4000]\) result in over-segmented images (Figs.~\ref{fig_sub:decide_number_parameters_1835} to \ref{fig_sub:decide_number_parameters_425}). The segmentation into 151 regions (Fig.~\ref{fig_sub:decide_number_parameters_151}) reveals closely spaced, potentially redundant boundaries, which are resolved by the merging into 115 regions (Fig.~\ref{fig_sub:decide_number_parameters_115}). This 115-region segmentation preserves clear foreground–background separation and can therefore be viewed as optimal solution.

\subsubsection{Segmentation vs. oversegmentation} Next, we compare the results of the baseline oversegmentation algorithm Power-SLIC with those obtained after applying our merging procedure. Fig.~\ref{fig_sub:noise_image} presents a noisy grayscale image containing 25 disks, some of which are partially occluded. We initialized our method by generating 300 superpixels using Power-SLIC (Fig.~\ref{fig_sub:noise_in}); these superpixels conform to disk boundaries while retaining regular shapes in homogeneous regions. The algorithm then iteratively merged adjacent superpixels, yielding the intermediate results shown in Figs.~\ref{fig_sub:noise_r_200}-\ref{fig_sub:noise_r_100} and the final segmentation depicted in Fig.~\ref{fig_sub:noise_out_sp}. By contrast, considering only the raw superpixel output for 25 superpixels as segmentation is unsatisfactory (Fig.~\ref{fig_sub:noise_out_super}). 

\subsubsection{OT vs. variance} To assess the contribution of the OT-based metric in isolation, we substituted the superpixel initialization with a uniform rectangular partition and compared the outcome to that of the SMST algorithm. Since SMST also employs an initial superpixel decomposition followed by minimum-spanning-tree merging and variance-based splitting, any discrepancies can be attributed solely to the choice of metric.

Figure~\ref{fig:H&ESEG12_rectangle} shows, in its first row, two representative images from Dataset~1, overlaid with uniform rectangular partitions of varying resolution. Since the task was to segment red cells, only the red channels of the images were provided to both models. The second row of the figure presents the segmentation results from~SP. Despite the rectangular initialization, our approach captured foreground regions using the smallest boxes possible while adhering to the rectangle boundaries. The third row of Fig.~\ref{fig:H&ESEG12_rectangle} shows the SMST results, which detected fewer cells and, across all resolutions, tended to over-segment the background into multiple regions. 

\begin{figure}[htb]
\centering

\subfloat[I1, $40\times40$\label{fig_sub:he1_box_40_image}]{\includegraphics[width=0.249\linewidth]{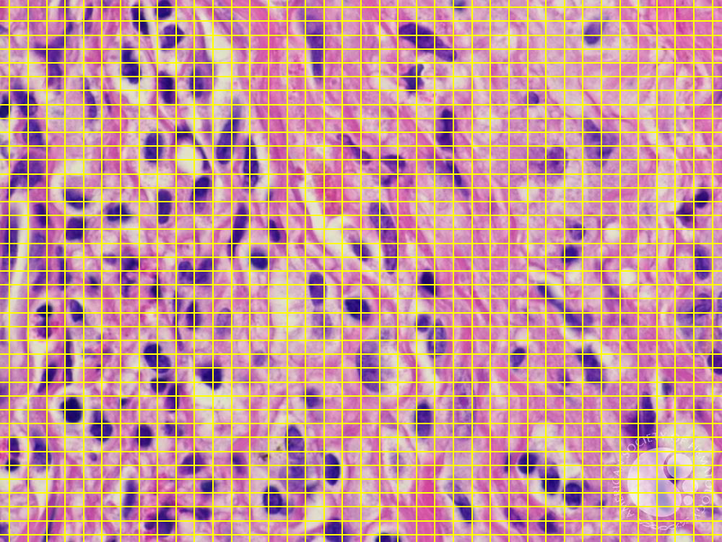}}
\hfill
\subfloat[I1, $30\times30$\label{fig_sub:he1_box_30_image}]{\includegraphics[width=0.249\linewidth]{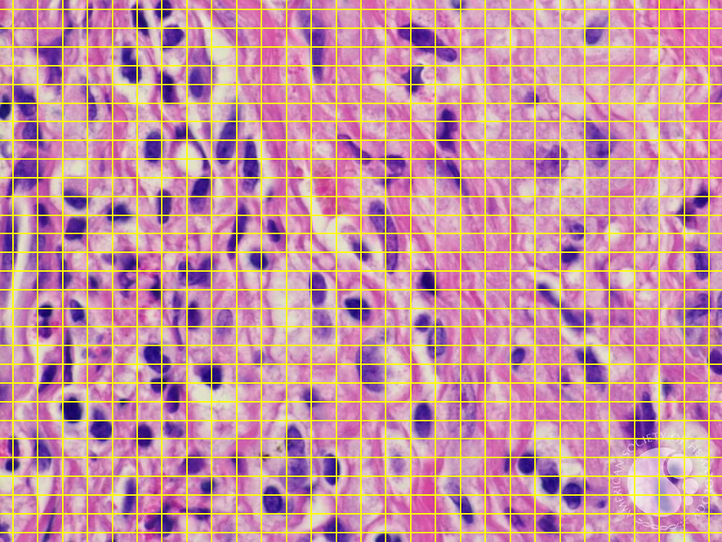}}
\hfill
\subfloat[I2, $80\times80$\label{fig_sub:he2_box_80_image}]{\includegraphics[width=0.249\linewidth]{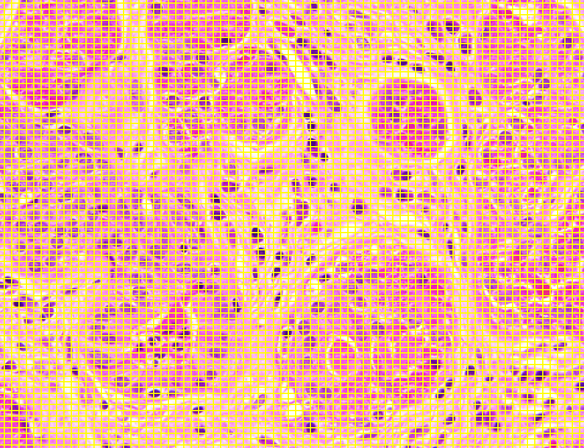}}
\hfill
\subfloat[I2, $70\times70$\label{fig_sub:he2_box_70_image}]{\includegraphics[width=0.249\linewidth]{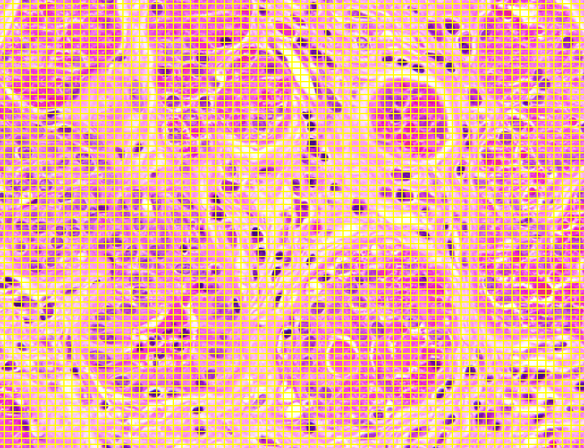}}

\subfloat[SP, I1, $40\times40$\label{fig_sub:he1_box_40_out_sp}]{\includegraphics[width=0.249\linewidth]{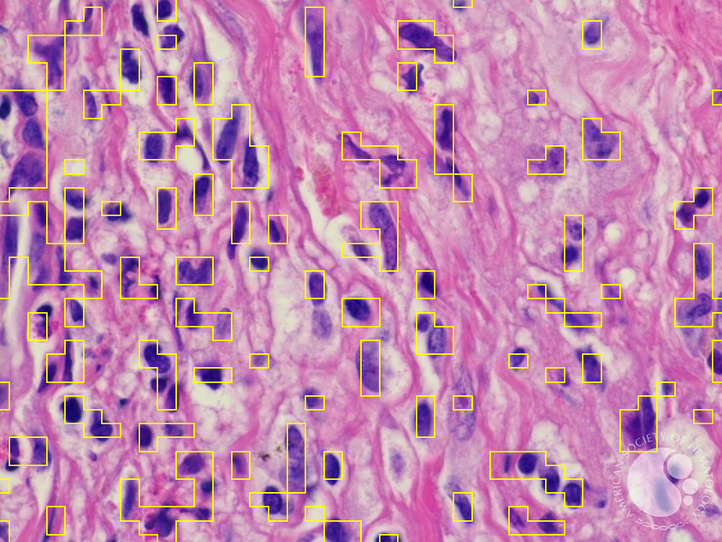}}
\hfill
\subfloat[SP, I1, $30\times30$\label{fig_sub:he1_box_30_out_sp}]{\includegraphics[width=0.249\linewidth]{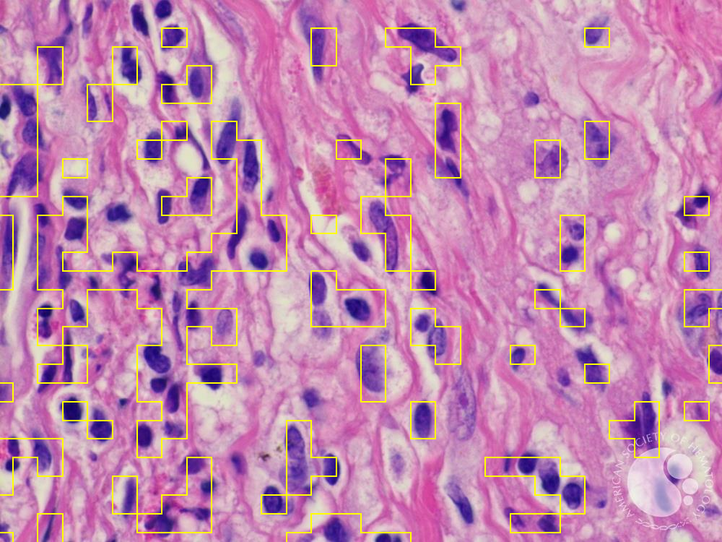}}
\hfill
\subfloat[SP, I2, $80\times80$\label{fig_sub:he2_box_80_out_sp}]{\includegraphics[width=0.249\linewidth]{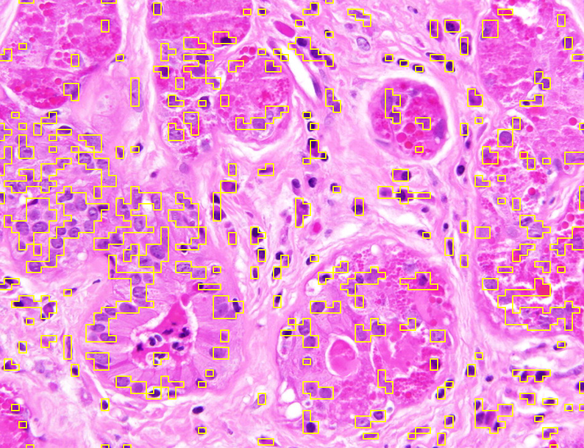}}
\hfill
\subfloat[SP, I2, $70\times70$\label{fig_sub:he2_box_70_out_sp}]{\includegraphics[width=0.249\linewidth]{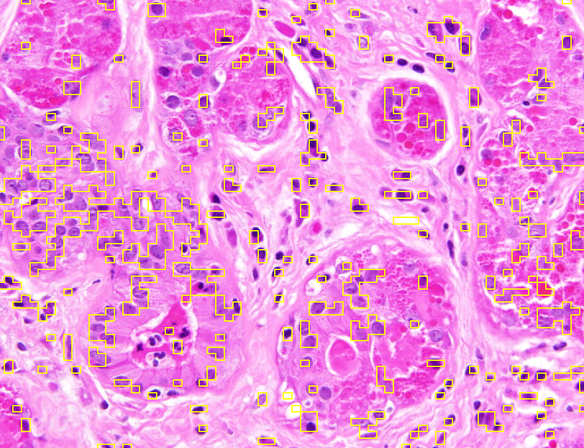}}

\subfloat[SMST, I1, $40\times40$\label{fig_sub:he1_box_40_out_mst}]{\includegraphics[width=0.249\linewidth]{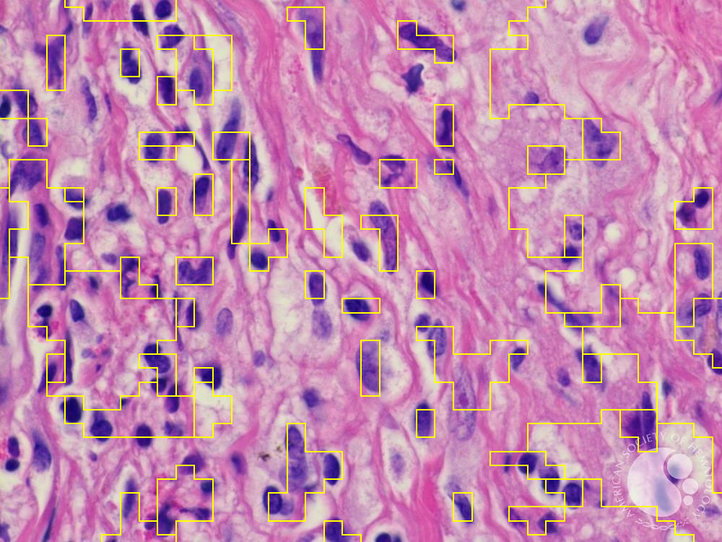}}
\hfill
\subfloat[SMST, I1, $30\times30$\label{fig_sub:he1_box_30_out_mst}]{\includegraphics[width=0.249\linewidth]{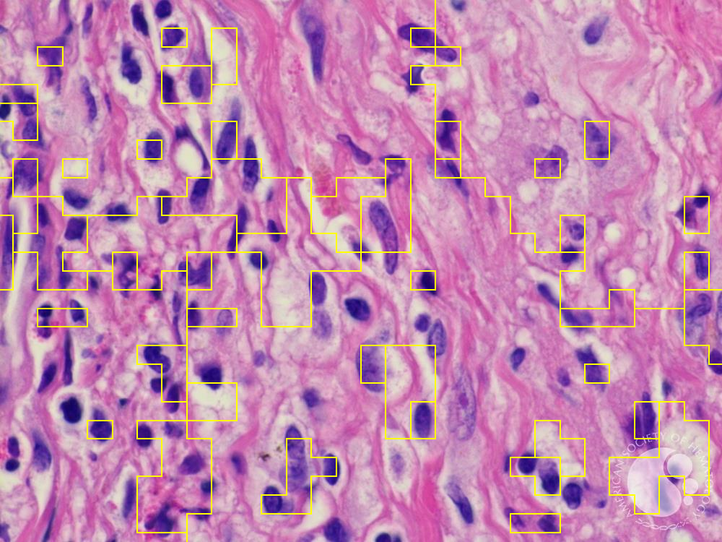}}
\hfill
\subfloat[SMST, I2, $80\times80$\label{fig_sub:he2_box_80_out_mst}]{\includegraphics[width=0.249\linewidth]{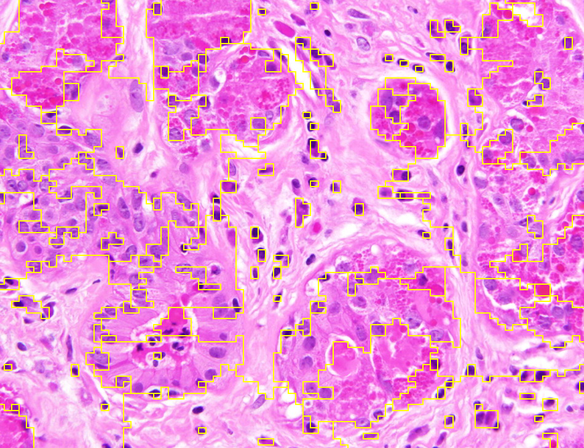}}
\hfill
\subfloat[SMST, I2, $70\times70$\label{fig_sub:he2_box_70_out_mst}]{\includegraphics[width=0.249\linewidth]{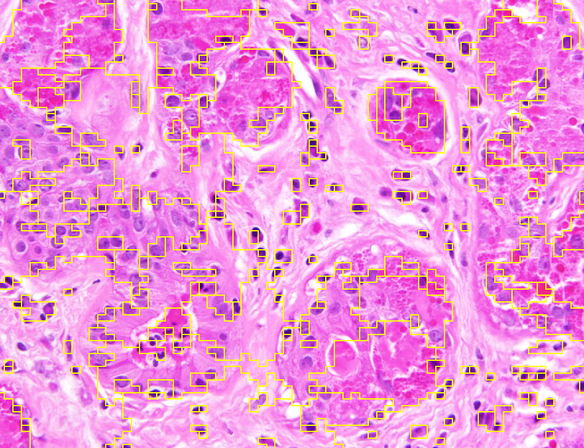}}
\caption{Comparison of OT-based segmentation and variance-based SMST. The first row (a–d) shows two images I1 and I2 from Dataset~1 partitioned into rectangular superpixels of sizes 40×40 (a), 30×30 (b), 80×80 (c), and 70×70 (d), respectively. The second row (e–h) shows SP results (times in seconds: 1.97, 0.28, 22.48, 16.12). The third row (i–l) shows SMST results (times in seconds: 0.35, 0.35, 0.37, 0.43).
}\label{fig:H&ESEG12_rectangle}
\end{figure}

\subsubsection{Selective models}
We now evaluate SP's selective segmentation (Section~\ref{SPmarker}) by specifying markers and their class membership. Several level set models, such as the Roberts–Chen model \cite{roberts2019convex} and the AR model, provide graphical interfaces that allow users to place seeds or draw regions marking the objects of interest. Some deep learning models, such as SAM, also support semantic selective segmentation. 

Figure~\ref{fig:selective_original_image} shows an image from Dataset 3 with two different sets of markers. In the first case, Fig.~\ref{fig:selective_original_image}(a,b), we manually placed five foreground seeds and one background seed; in the second case, Fig.~\ref{fig:selective_original_image}(c,d), we slightly perturb the position of makers to test the stability of different models.  Foreground and background seeds are labeled with the prefixes `f' and `b,' respectively, with the following number indicating their order (e.g., f2 denotes the second foreground seed). Note that there is only one background seed b1 in both cases. We then provided these to AR, SAM, and SP as markers.  

Our model SP follows the selective variant from Section~\ref{SPmarker}, with the target region number set equal to the number of markers. For SAM, we used the selective segmentation module in MATLAB 2024b. The AR model, while not requiring explicit background seeds, still needed background information; this was provided by assuming the domain outside the specified foreground was background. Since AR expects polygonal regions rather than points, we dilated the foreground seeds with a disk kernel and used the resulting regions as markers.

Figure~\ref{fig:selective2} shows the results for the first case (Fig.~\ref{fig_sub:selective2_seeds1}) with incrementally added foreground markers. SAM performed well with 1-2 markers, but with 4-5 it incorrectly classified much of the domain as foreground suggesting that the scattered markers led SAM to infer that the entire image was of interest. AR produced a clear background but failed to segment the cells containing f1, f3, and f5, while adding an extra cell around f4. By contrast, SP consistently captured all marked cells while maintaining a clean background.

Figure~\ref{fig:selective2_another} shows the results for the second case, the case with slightly perturbed marker positions  (Fig.~\ref{fig_sub:selective_seeds3}). Similar to the previous scenario, SAM performed successfully only on the cell with the first foreground marker, while for the other marker, SAM mixed the target cell with other structures. AR showed  similar problems as before, recognising only the first two cells (f1 and f2) and disregarding the others. This occurred because, during iterations, AR incorrectly shrank the level set corresponding to the other markers. The results for SP are identically to those in Fig.~\ref{fig:selective2} and are therefore not dublicated in this figure.

In summary, AR and SAM seem sensitive to both the position and number of markers, whereas SP yields consistent results despite marker perturbations, benefiting from noise-robust superpixels and an OT metric that is comparatively insensitive to such variability.

\begin{figure}[htb]
\centering

\subfloat[\label{fig_sub:selective2_seeds1}]{\includegraphics[width=0.249\linewidth]{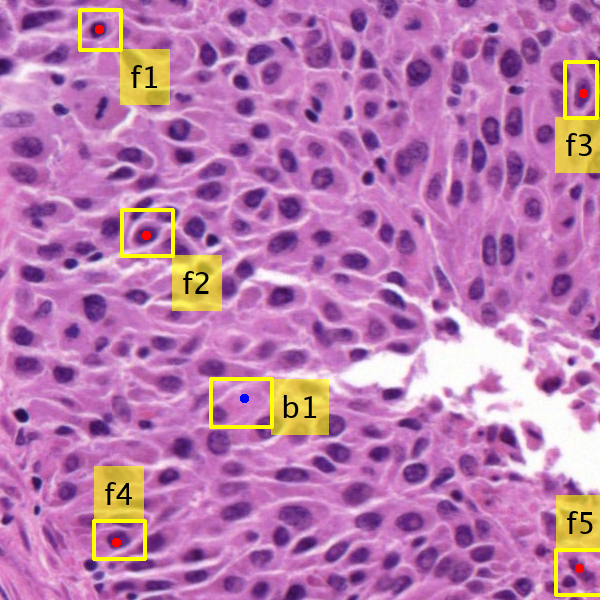}}
\hfill
\subfloat[\label{fig_sub:selective2_seeds2}]{\includegraphics[width=0.249\linewidth]{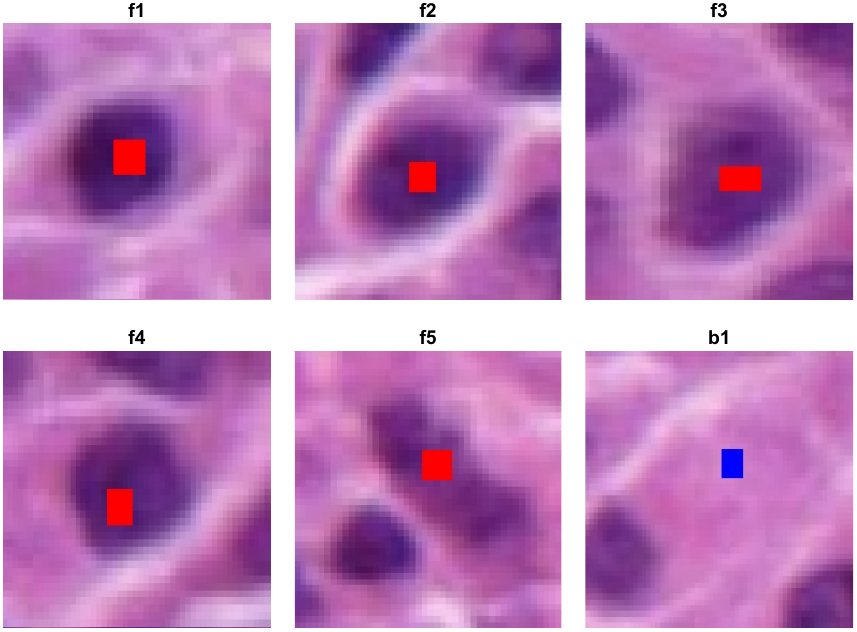}}
\hfill
\subfloat[\label{fig_sub:selective_seeds3}]{\includegraphics[width=0.249\linewidth]{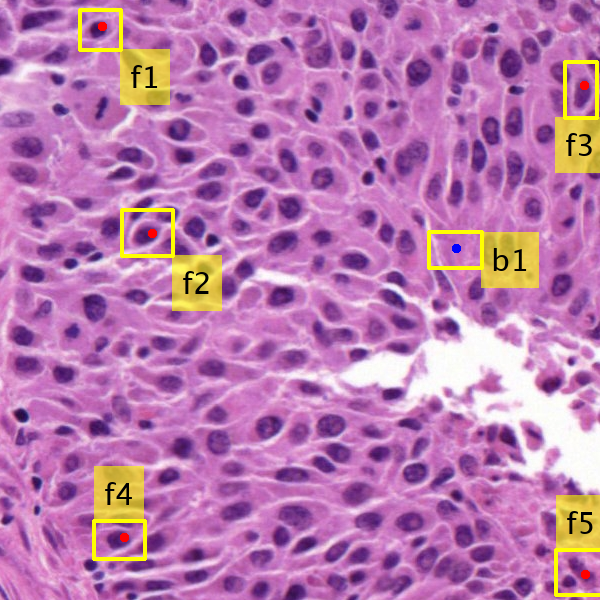}}
\hfill
\subfloat[\label{fig_sub:selective_seeds4}]{\includegraphics[width=0.249\linewidth]{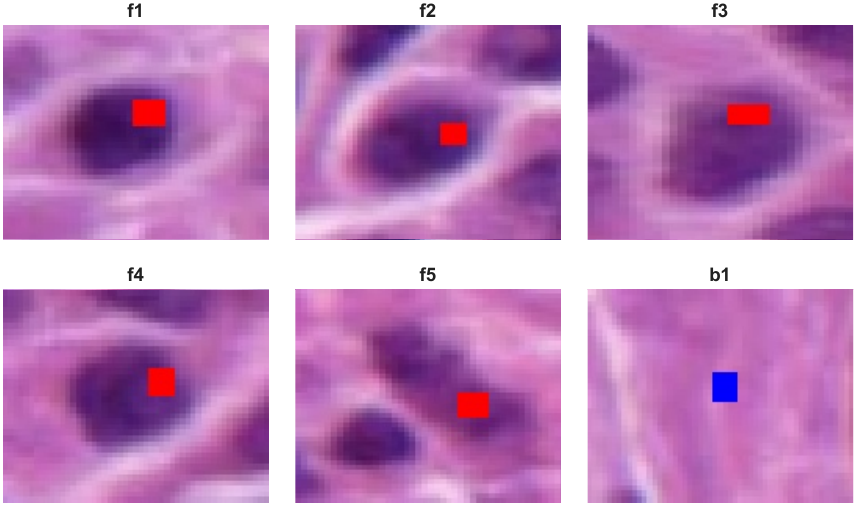}}

\caption{Selective segmentation seeds: (a,c) images with foreground (red) and background (blue) seeds, labeled ‘f’ and ‘b’ respectively (numbers indicate order); (b,d) corresponding zoomed-in views.
}\label{fig:selective_original_image}
\end{figure}

\begin{figure}[htb]
\centering

\subfloat[SAM, f1]{\includegraphics[width=0.249\linewidth]{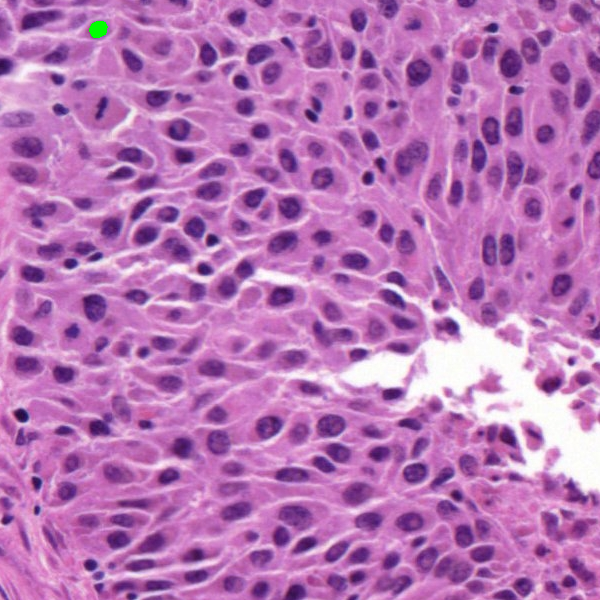}}
\hfill
\subfloat[SAM, f1-f2]{\includegraphics[width=0.249\linewidth]{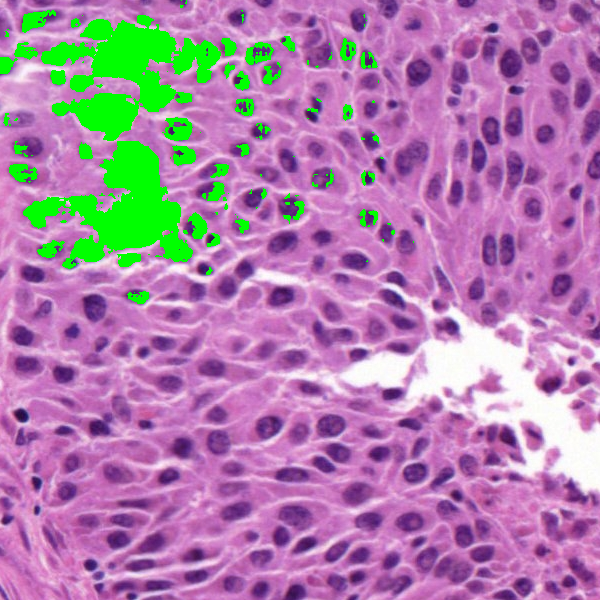}}
\hfill
\subfloat[SAM, f1-f3]{\includegraphics[width=0.249\linewidth]{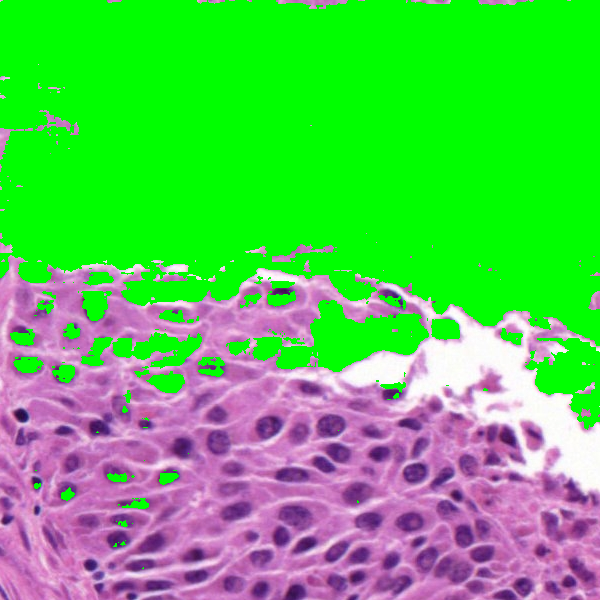}}
\hfill
\subfloat[SAM, f1-f4]{\includegraphics[width=0.249\linewidth]{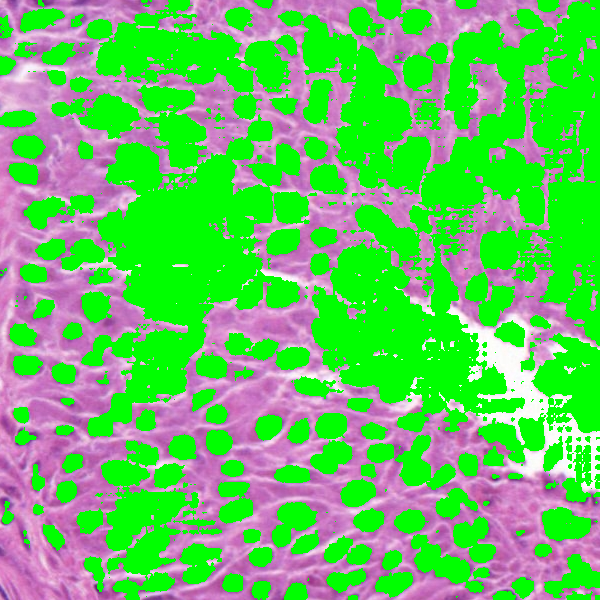}}

\subfloat[SAM, f1-f5]{\includegraphics[width=0.249\linewidth]{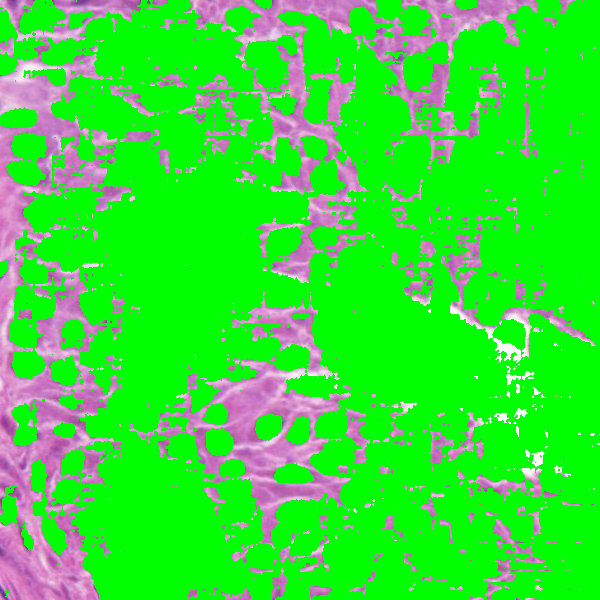}}
\hfill
\subfloat[AR, f1]{\includegraphics[width=0.249\linewidth]{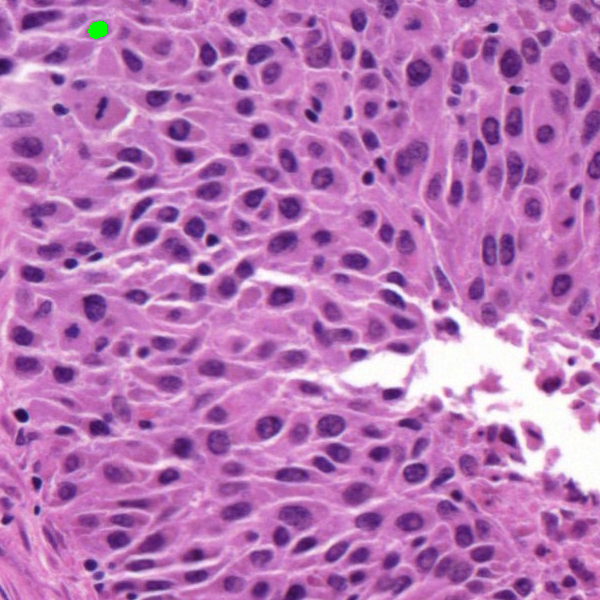}}
\hfill
\subfloat[AR, f1-f2 ]{\includegraphics[width=0.249\linewidth]{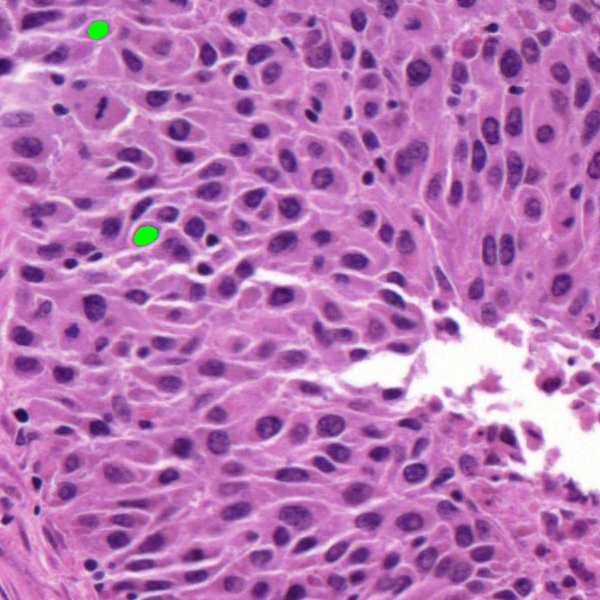}}
\hfill
\subfloat[AR, f1-f3]{\includegraphics[width=0.249\linewidth]{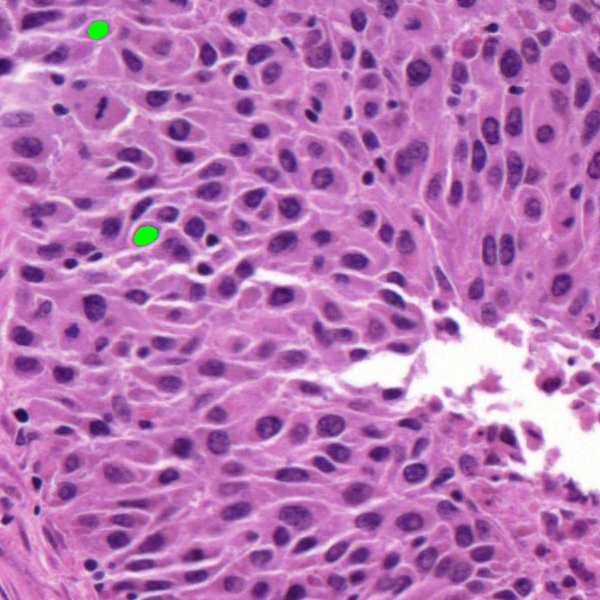}}

\subfloat[AR, f1-f4]{\includegraphics[width=0.249\linewidth]{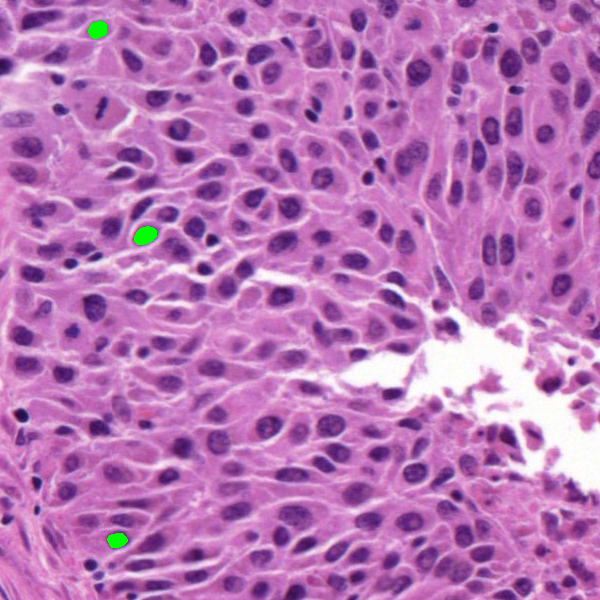}}
\hfill
\subfloat[AR, f1-f5]{\includegraphics[width=0.249\linewidth]{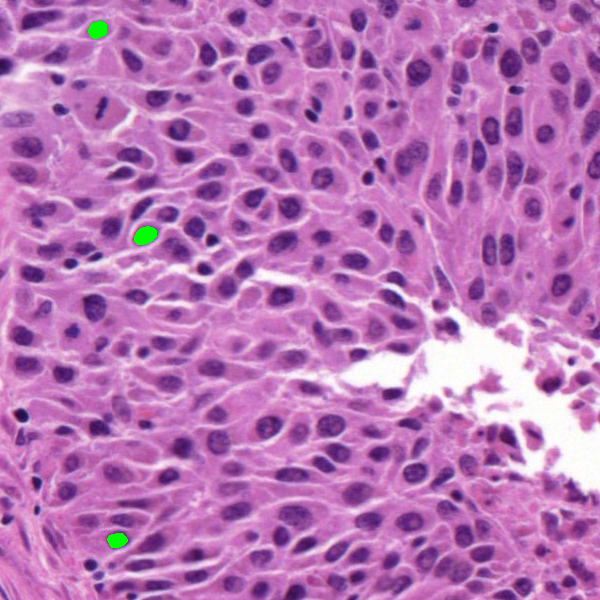}}
\hfill
\subfloat[SP, f1\label{fig_sub:selective2_our1}]{\includegraphics[width=0.249\linewidth]{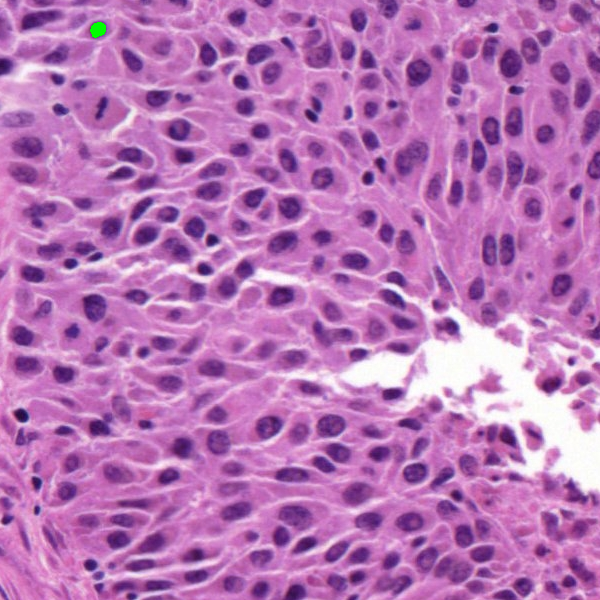}}
\hfill
\subfloat[SP, f1-f2\label{fig_sub:selective2_our2}]{\includegraphics[width=0.249\linewidth]{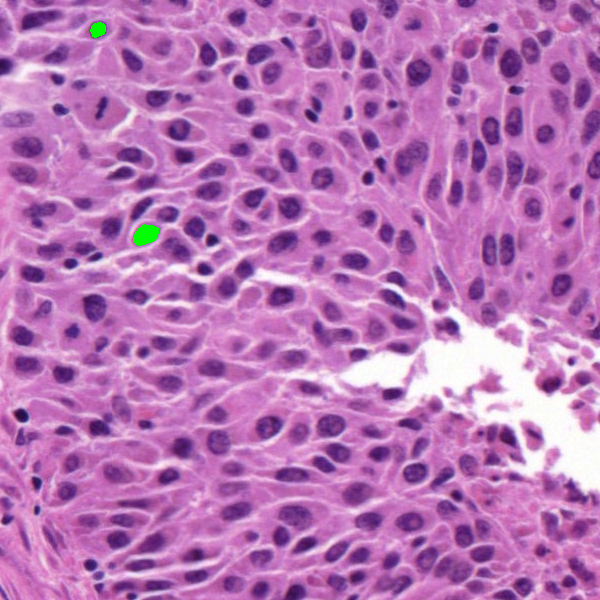}}

\subfloat[SP, f1-f3\label{fig_sub:selective2_our3}]{\includegraphics[width=0.249\linewidth]{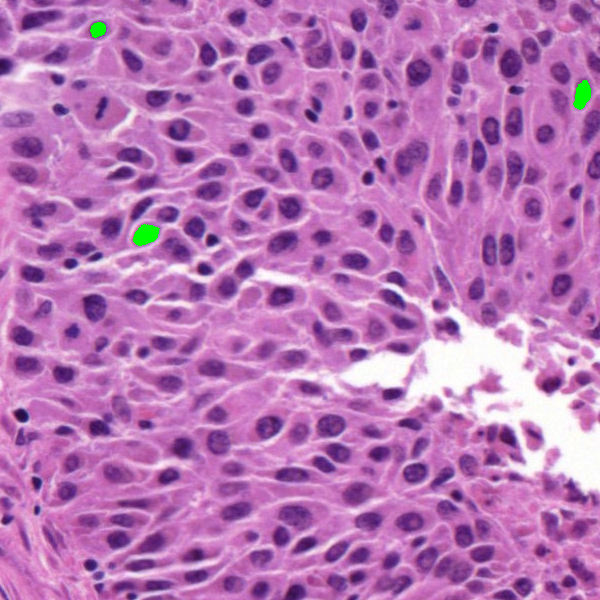}}
\hfill
\subfloat[SP, f1-f4\label{fig_sub:selective2_our4}]{\includegraphics[width=0.249\linewidth]{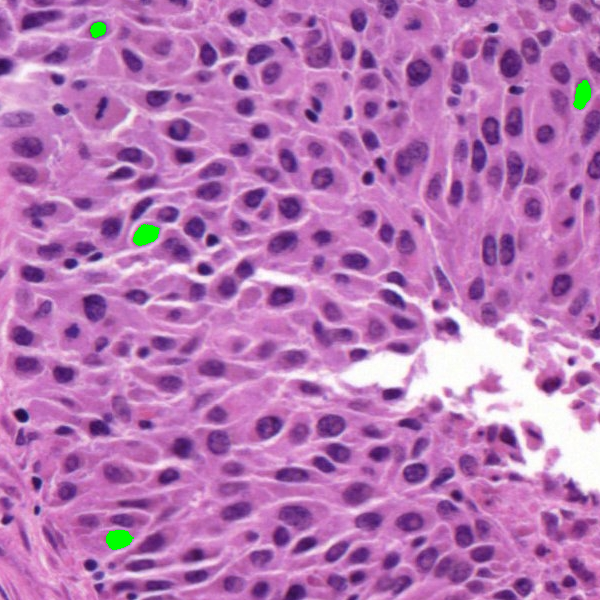}}
\hfill
\subfloat[SP, f1-f5\label{fig_sub:selective2_our5}]{\includegraphics[width=0.249\linewidth]{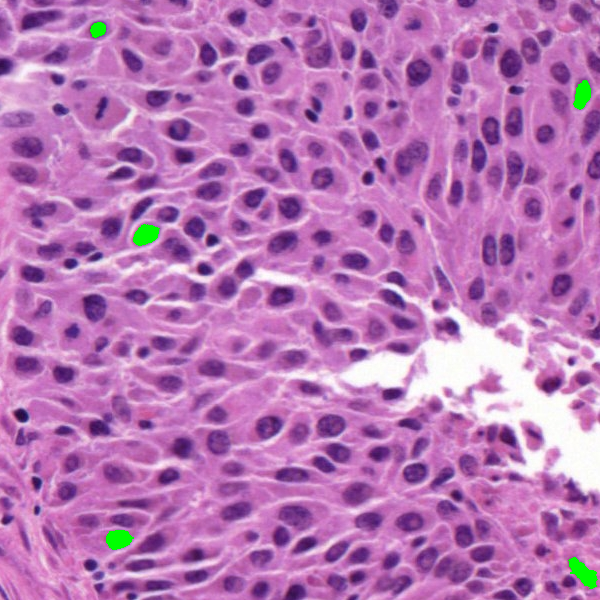}}
\hfill
\hspace{0.249\linewidth}
\caption{Selective segmentation results for the image in Fig.~\ref{fig_sub:selective2_seeds1} using background marker b1 and incrementally added foreground markers: f1, f1-f2, f1-f3, f1-f4, and f1-f5. (a–e) SAM, (f–j) AR, (k–o) SP.}\label{fig:selective2}
\end{figure}

\begin{figure}[htb]
\centering

\subfloat[SAM, f1]{\includegraphics[width=0.249\linewidth]{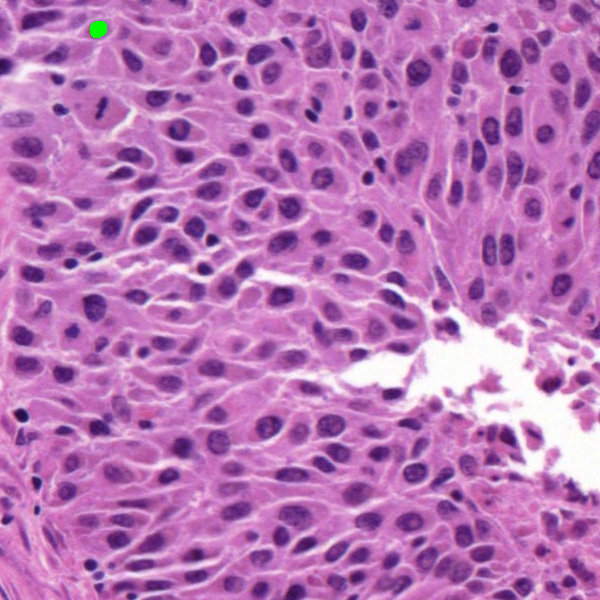}}
\hfill
\subfloat[SAM, f1-f2]{\includegraphics[width=0.249\linewidth]{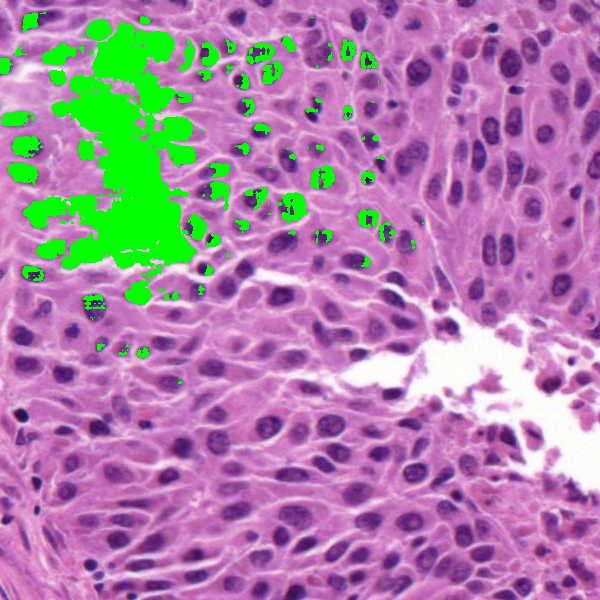}}
\hfill
\subfloat[SAM, f1-f3]{\includegraphics[width=0.249\linewidth]{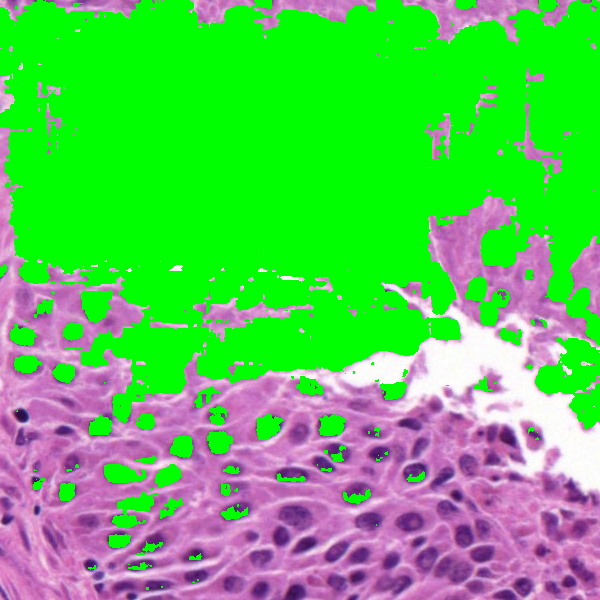}}
\hfill
\subfloat[SAM, f1-f4]{\includegraphics[width=0.249\linewidth]{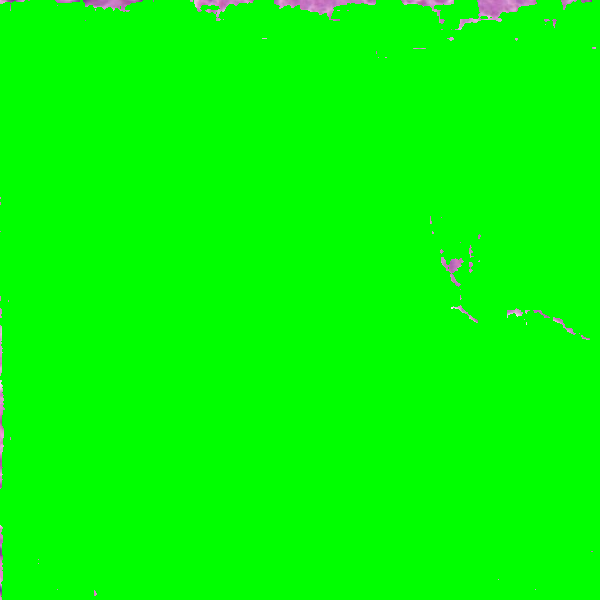}}

\subfloat[SAM, f1-f5]{\includegraphics[width=0.249\linewidth]{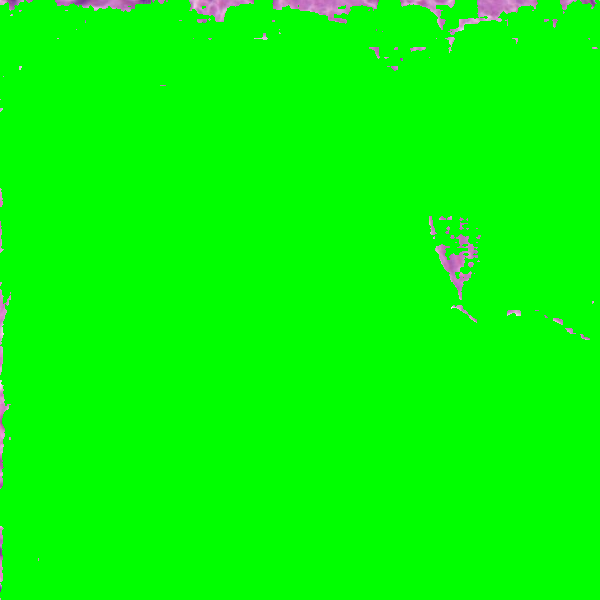}}
\hfill
\subfloat[AR, f1]{\includegraphics[width=0.249\linewidth]{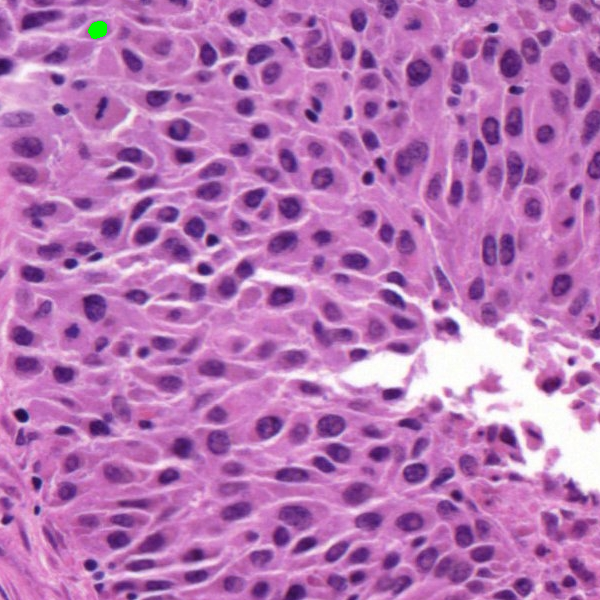}}
\hfill
\subfloat[AR, f1-f2 ]{\includegraphics[width=0.249\linewidth]{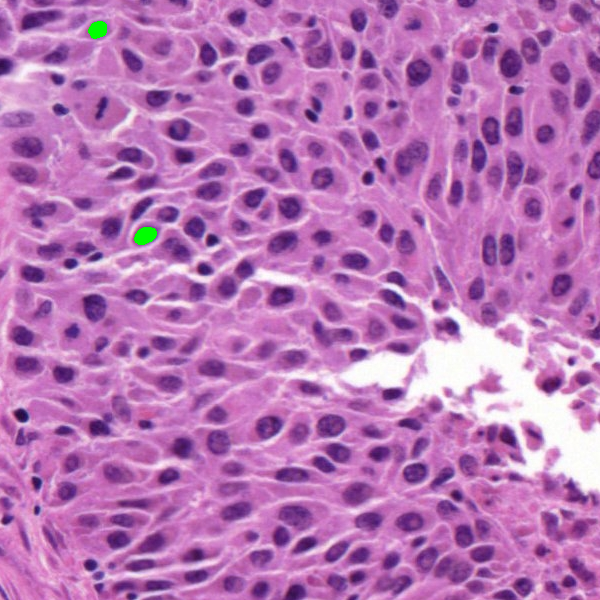}}
\hfill
\subfloat[AR, f1-f3]{\includegraphics[width=0.249\linewidth]{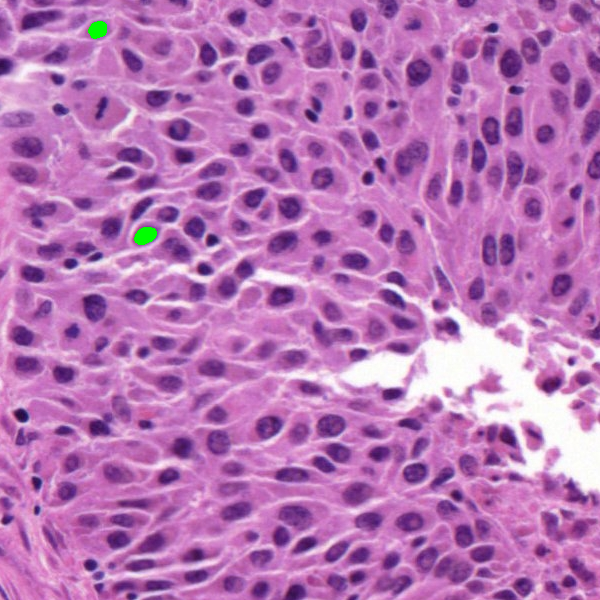}}

\subfloat[AR, f1-f4]{\includegraphics[width=0.249\linewidth]{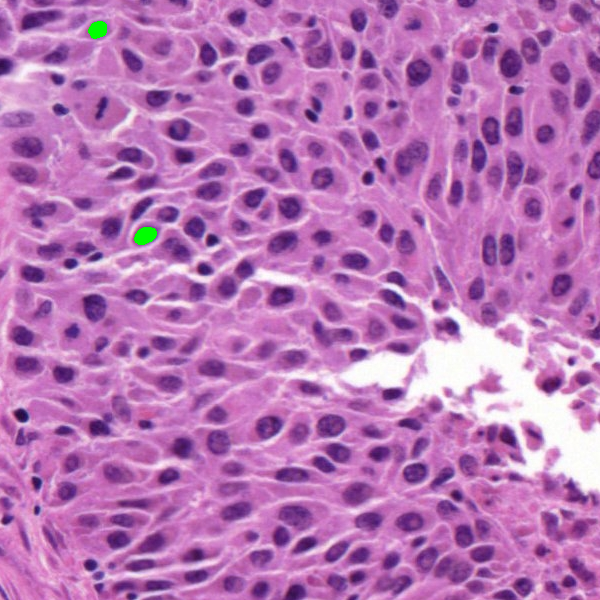}}
\hfill
\subfloat[AR, f1-f5]{\includegraphics[width=0.249\linewidth]{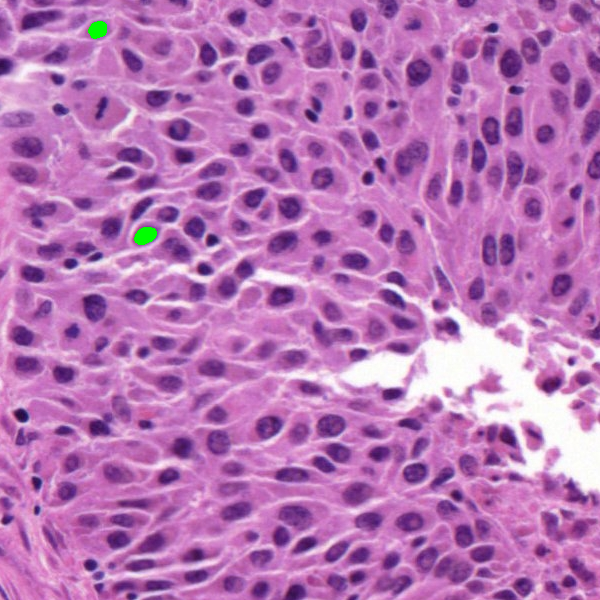}}

\caption{Selective segmentation results for the image in Fig.~\ref{fig_sub:selective_seeds3} using background marker b1 and incrementally added foreground markers: f1, f1-f2, f1-f3, f1-f4, and f1-f5. (a–e) SAM, (f–j) AR.}\label{fig:selective2_another}
\end{figure}

\subsubsection{Unsupervised models} We validate unsupervised SP on the three datasets by comparison with SAM, SMST, AR, and ZZ. All methods are fully unsupervised; however, AR and ZZ require an initial level set for contour evolution.

We first present  segmentation results for Dataset~1. Figures~\ref{fig_sub:rending_output_sam}-\ref{fig_sub:rending_output_SP} show the outputs for the image in Fig.~\ref{fig_sub:rending_image}, for SAM, AR, ZZ, SMST, and SP. The AR and ZZ models used manually specified initial level sets (Fig.~\ref{fig_sub:rending_input_AR}), while SP and SMST used the same set of 4000 superpixels generated by Power-SLIC (Fig.~\ref{fig_sub:rending_input_sp}) using the first two LAB channels. SAM relied solely on RGB images. The segmentation goal was to separate the red mesh from the grayscale background. SAM omitted some meshes and introduced background noise; AR and ZZ misclassified many dark meshes as background due to their reliance on local region terms; the SMST method erroneously classified portions of the foreground as background and produced a fragmented representation of the background region. In contrast, SP  achieved accurate segmentation of all meshes while successfully suppressing the background. Computation times are as follows: SMST was fastest (0.42s + 0.38s for Power-SLIC and merging), followed by SP (0.42s + 4.45s for Power-SLIC and merging). SAM required~14.62s (including model loading), while ZZ was the slowest at 11102.89s, requiring a substantial number of pixel level convolutions and  iterations in the employed additive operator splitting (AOS) scheme; see~\cite{roberts2019convex}.

\begin{figure}[htb]
\centering
\subfloat[Initial\label{fig_sub:rending_input_AR}]{\includegraphics[width=0.249\linewidth]{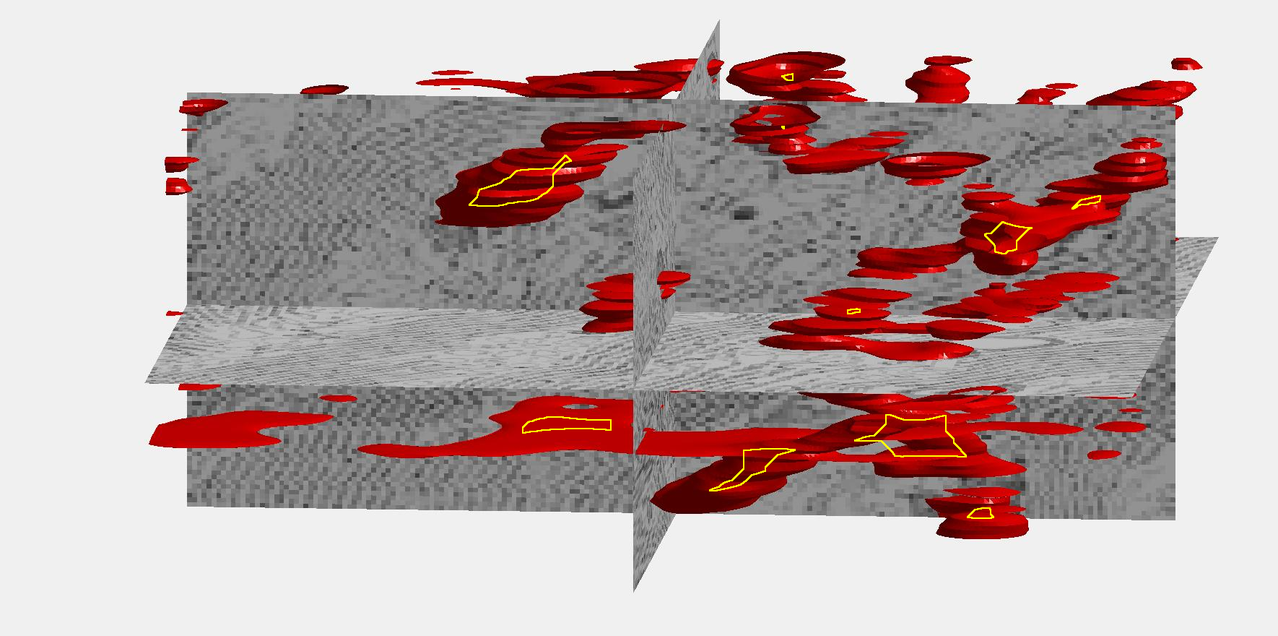}}
\hfill
\subfloat[Initial\label{fig_sub:rending_input_sp}]{\includegraphics[width=0.249\linewidth]{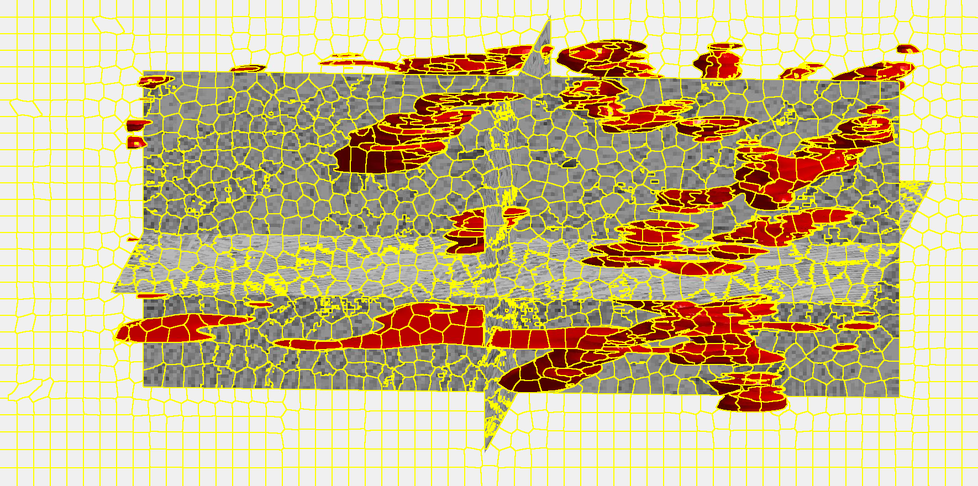}}
\hfill
\subfloat[SAM\label{fig_sub:rending_output_sam}]{\includegraphics[width=0.249\linewidth]{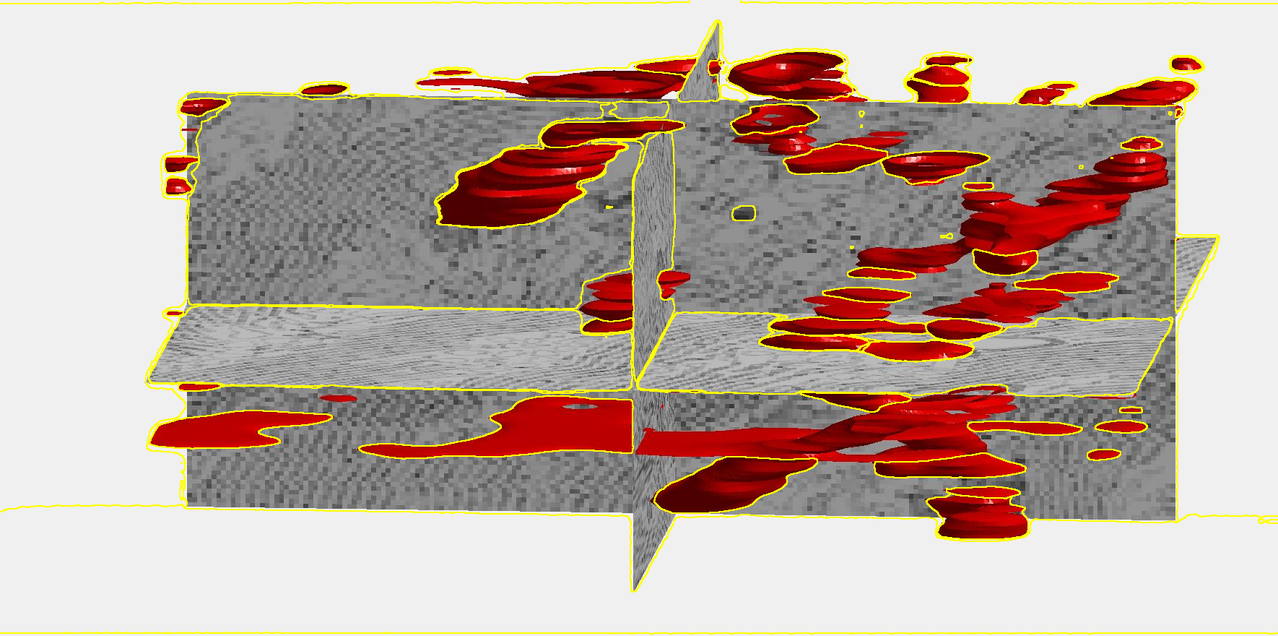}}
\\
\subfloat[AR\label{fig_sub:rending_output_AR}]{\includegraphics[width=0.249\linewidth]{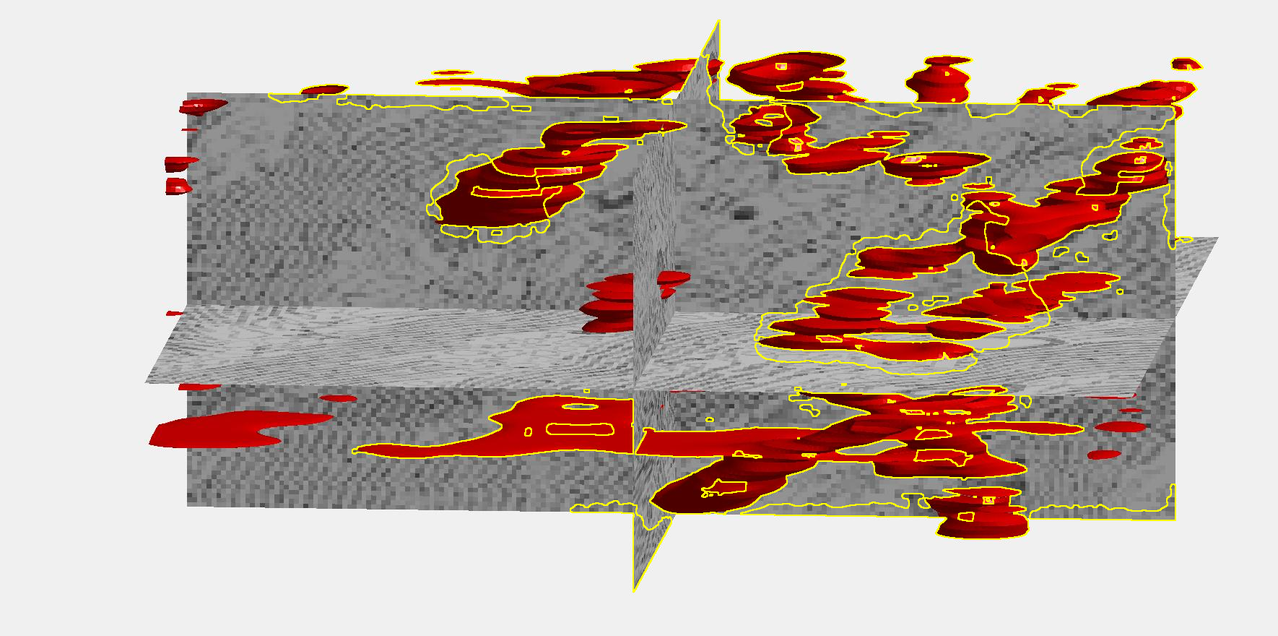}}
\hfill
\subfloat[ZZ\label{fig_sub:rending_output_ZZ}]{\includegraphics[width=0.249\linewidth]{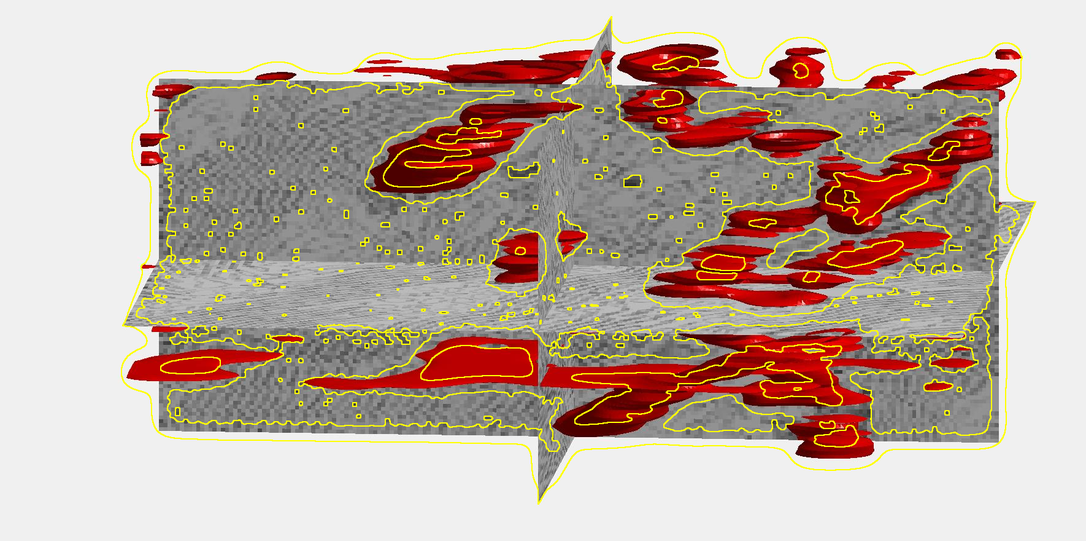}}
\hfill
\subfloat[SMST\label{fig_sub:rending_output_MST}]{\includegraphics[width=0.249\linewidth]{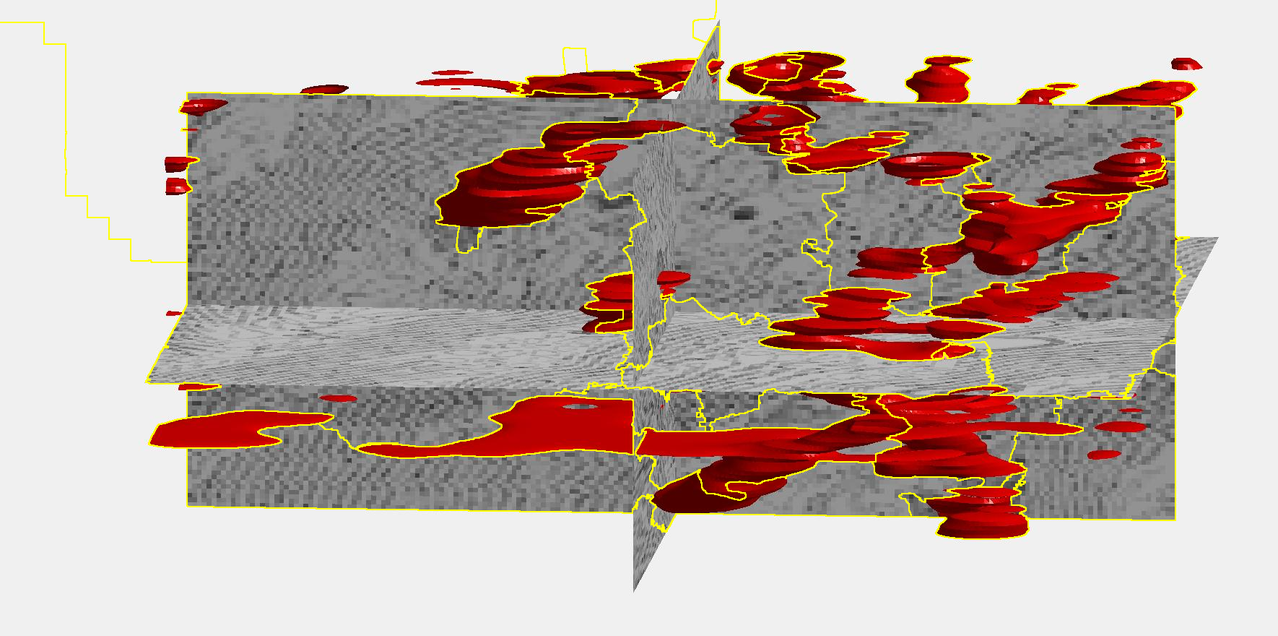}}
\hfill
\subfloat[SP\label{fig_sub:rending_output_SP}]{\includegraphics[width=0.249\linewidth]{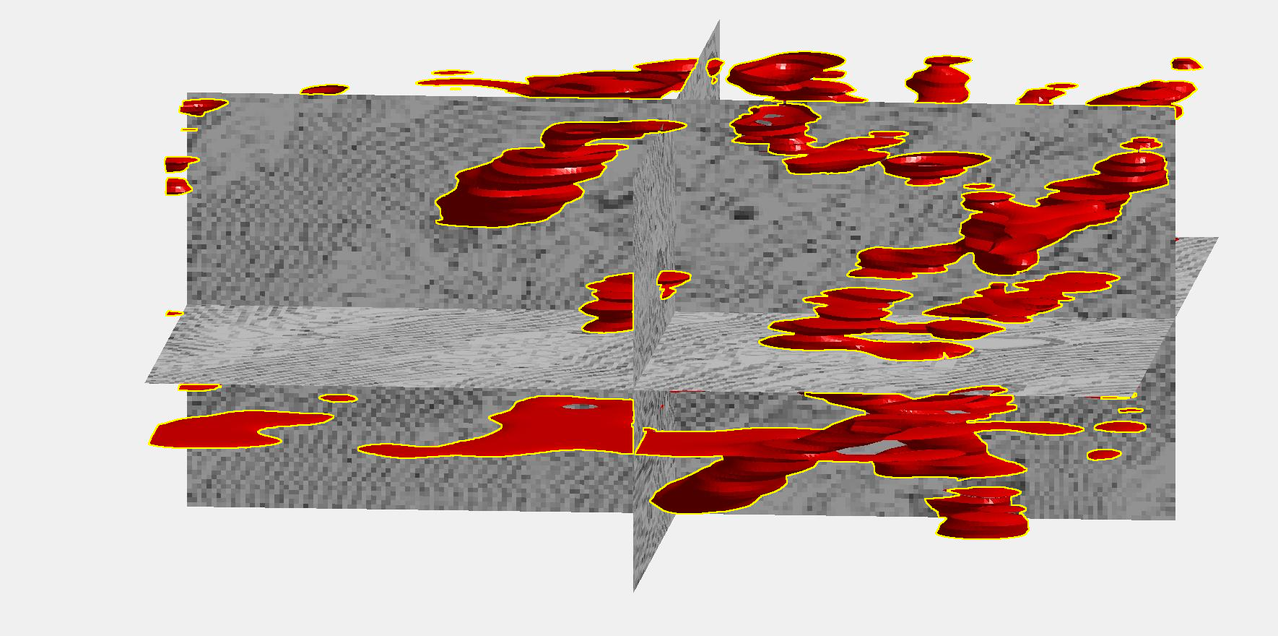}}
\caption{Comparison of unsupervised segmentation models on Dataset~1. (a) and (b) show the intial level sets and superpixel, respectively: AR and ZZ in (a) (yellow curves, computation times~$\infty$), SMST and SP in (b) (yellow curves, 0.42s).
(c–g) show segmentation results (yellow curves) from SAM, AR, ZZ, SMST, and SP, respectively.}\label{fig:3D_render_seg}
\end{figure}
Figure~\ref{fig:H&ESEG1} shows the  segmentation results for the first hematoxylin and eosin (H\&E)–stained image from Dataset~1 (see Fig.~\ref{fig_sub:he1_image}).  Initial level sets and superpixels are shown in Figs.~\ref{fig_sub:he1_in1_zz}-\ref{fig_sub:he1_in_SP}: AR and ZZ used two distinct level sets for nuclei (S1) and cytoplasm (S2)  (Figs.~\ref{fig_sub:he1_in1_zz}-\ref{fig_sub:he1_in2_zz}), while SP and SMST (Fig.~\ref{fig_sub:he1_in_SP}) used the same set of 4000 Power-SLIC superpixels in the same color space (red channel for S1, remaining channels for S2). SAM does not allow separate-channel segmentation; its combined S1/S2 result is shown in Fig.~\ref{fig_sub:he1_out_SAM}. Figures~\ref{fig_sub:he1_out1_AR}-\ref{fig_sub:he1_out1_SP} display S1 results for AR, ZZ, SMST, and SP, and Figs.~\ref{fig_sub:he1_out2_AR}-\ref{fig_sub:he1_out2_SP} show the corresponding S2 results.

AR seemed highly sensitive to initialization, segmenting mainly regions near seed points and requiring seeds in nearly all connected foreground components for satisfactory performance. ZZ segmented a wider range of objects from fewer seeds but incorrectly included nuclei in cytoplasm segmentation, failing to capture white regions and introducing spurious tissue. SAM automatically detected many nuclei and some cytoplasm but merged all structures into a single segmentation, limiting its practical utility here. SMST was the fastest but failed to recover large homogeneous regions, such as the background, and frequently fragmented them due to its variance-based metric. In contrast, SP yielded the most accurate results on both S1 and S2, with clear background separation and precise boundary delineation.

\begin{figure}[htb]
\centering

\subfloat[S1 LS: AR, ZZ\label{fig_sub:he1_in1_zz}]{\includegraphics[width=0.249\linewidth]{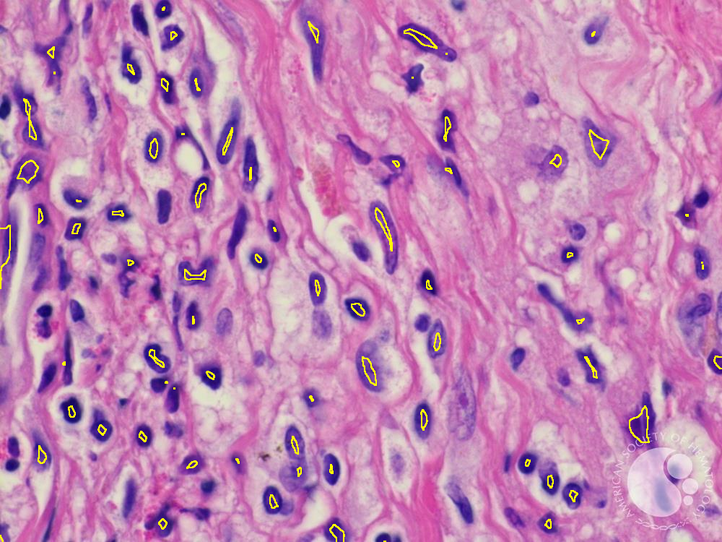}}
\hfill
\subfloat[S2 LS: AR, ZZ\label{fig_sub:he1_in2_zz}]{\includegraphics[width=0.249\linewidth]{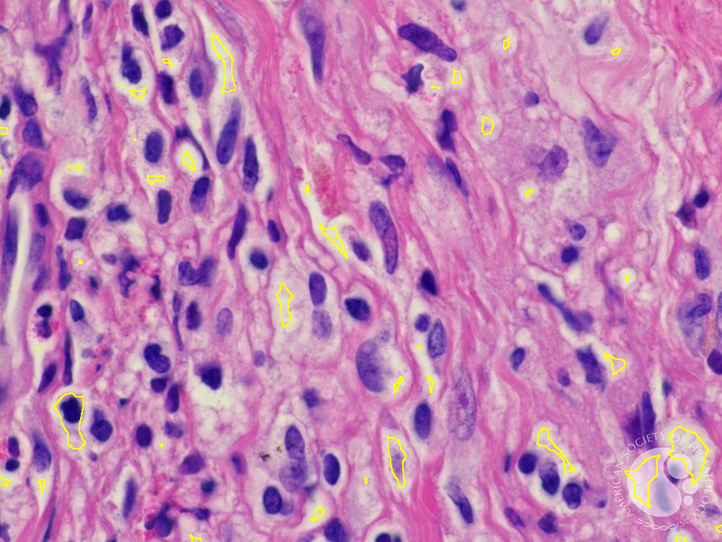}}
\hfill
\subfloat[S1/S2 SPx: SMST, SP\label{fig_sub:he1_in_SP}]{\includegraphics[width=0.249\linewidth]{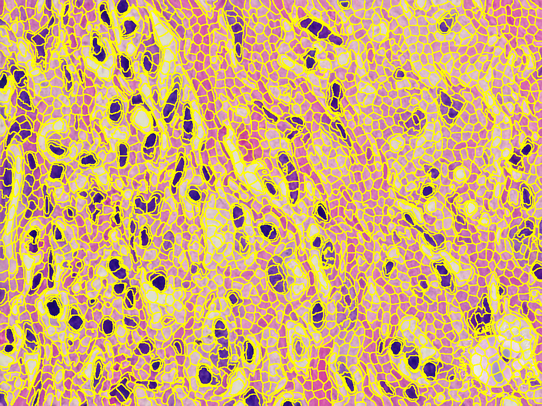}}
\hfill
\subfloat[S1/S2,  SAM\label{fig_sub:he1_out_SAM}]{\includegraphics[width=0.249\linewidth]{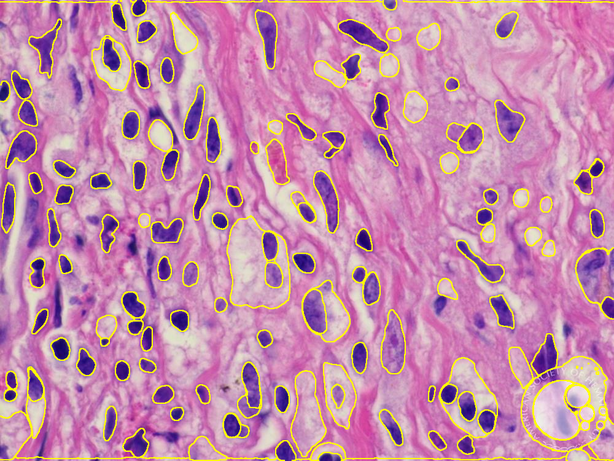}}

\subfloat[S1, AR\label{fig_sub:he1_out1_AR}]{\includegraphics[width=0.249\linewidth]{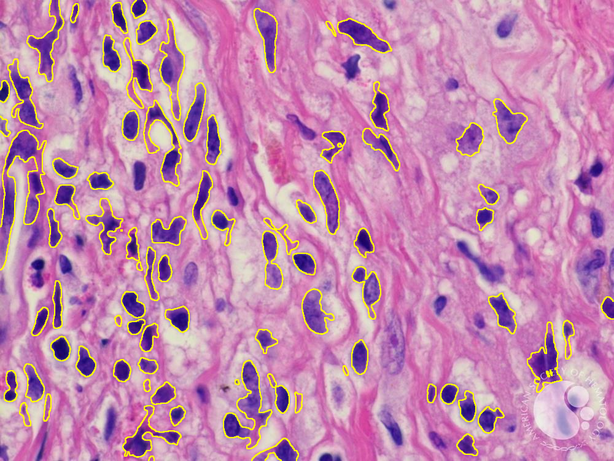}}
\hfill
\subfloat[S1, ZZ\label{fig_sub:he1_out1_ZZ}]{\includegraphics[width=0.249\linewidth]{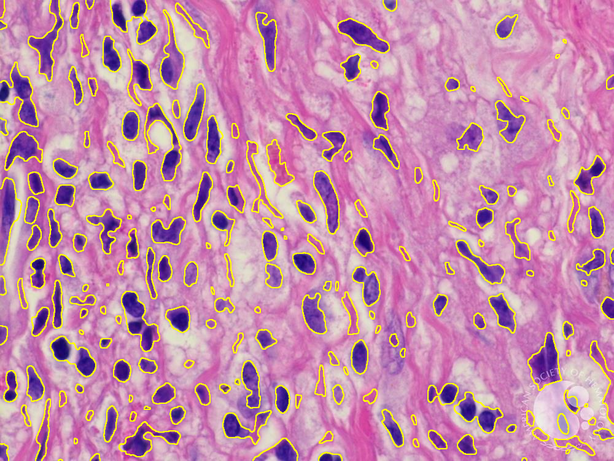}}
\hfill
\subfloat[S1,  SMST\label{fig_sub:he1_out1_MST}]{\includegraphics[width=0.249\linewidth]{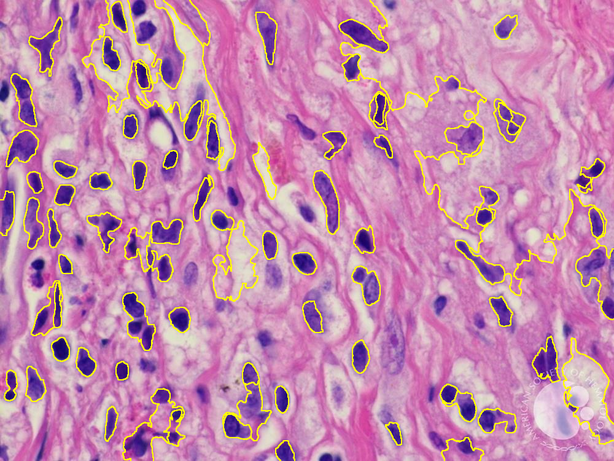}}
\hfill
\subfloat[S1, SP\label{fig_sub:he1_out1_SP}]{\includegraphics[width=0.249\linewidth]{parameter_setting/seg115.png}}

\subfloat[S2, AR\label{fig_sub:he1_out2_AR}]{\includegraphics[width=0.249\linewidth]{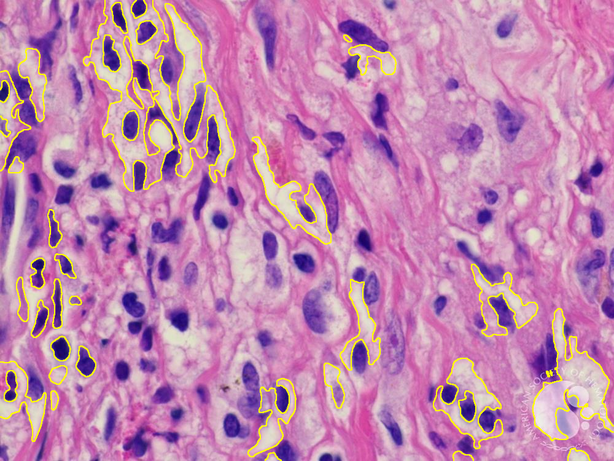}}
\hfill
\subfloat[S2, ZZ\label{fig_sub:he1_out2_ZZ}]{\includegraphics[width=0.249\linewidth]{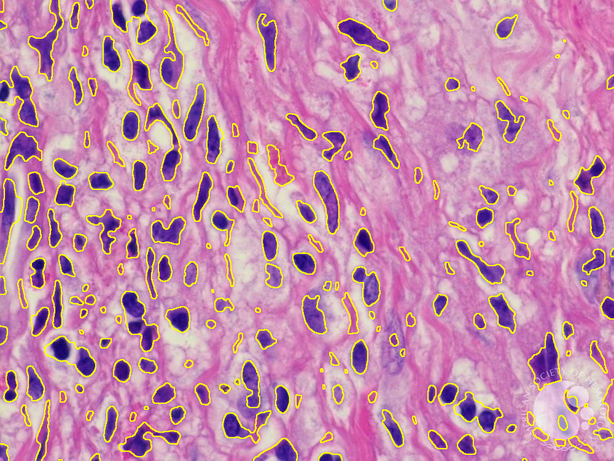}}
\hfill
\subfloat[S2, SMST\label{fig_sub:he1_out2_MST}]{\includegraphics[width=0.249\linewidth]{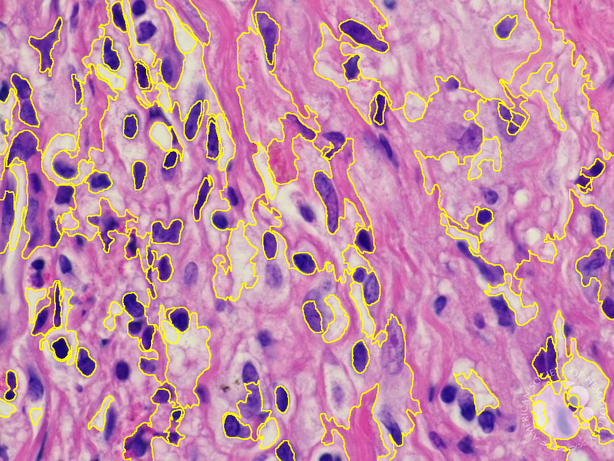}}
\hfill
\subfloat[S2, SP\label{fig_sub:he1_out2_SP}]{\includegraphics[width=0.249\linewidth]{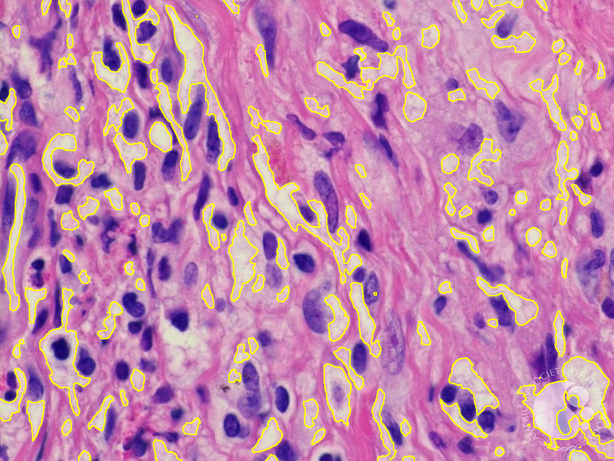}}

\caption{Unsupervised segmentation of the image in Fig.~\ref{fig_sub:he1_image} for two structures: purple cells (S1) and white cytoplasm (S2). Initializations via level sets (LS) for AR and ZZ are shown in (a, b); 4000 superpixels for SMST and SP in (c) (red channel for S1, remaining  channels for S2). SAM does not allow separate-channel segmentation; its overall output is in (d). Results for S1 are in (e–h) and for S2 in (i–l). Computation times (S1/S2, in seconds): AR 829.51/570.32, ZZ 4304.50/4470.56, SMST 0.38/0.38, SP 5.98/4.26; superpixels in (c) computed in 0.17s, SAM in (d) in 7.42s.
}\label{fig:H&ESEG1}
\end{figure}

Figure~\ref{fig:H&ESEG2} shows the results for the second hematoxylin–eosin stained image from Dataset~1 (see Fig.~\ref{fig_sub:he2_image}).  The AR model detected only structures near the initialization, missing many objects. The ZZ model recovered most cells but also introduced numerous unnecessary pink regions, and its two outputs were nearly identical despite different initial level sets (Figs.~\ref{fig_sub:he2_in1_AR}-\ref{fig_sub:he2_in2_AR}), similar to the over-segmentation behavior of SAM (Fig.~\ref{fig_sub:he2_out_SAM}). SAM, which does not require initialization, failed to separate different structures, while SMST produced too few foreground objects and added spurious regions, without a clear background. In contrast, SP, initialized with 6000 Power-SLIC superpixels (Fig.~\ref{fig_sub:he2_in_SP}), achieved accurate separation of nuclei and cytoplasm.

\begin{figure}[htb]
\centering

\subfloat[S1 LS: AR, ZZ\label{fig_sub:he2_in1_AR}]{\includegraphics[width=0.249\linewidth]{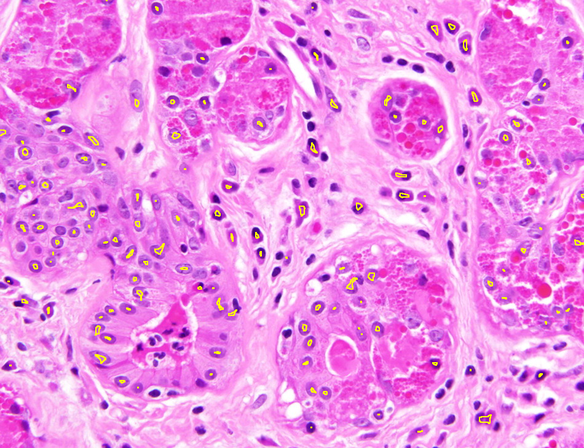}}
\hfill
\subfloat[S2 LS: AR, ZZ\label{fig_sub:he2_in2_AR}]{\includegraphics[width=0.249\linewidth]{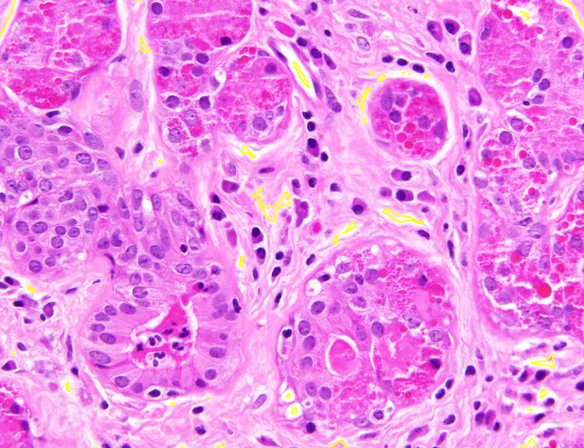}}
\hfill
\subfloat[S1/S2 SPx: SMST, SP \label{fig_sub:he2_in_SP}]{\includegraphics[width=0.249\linewidth]{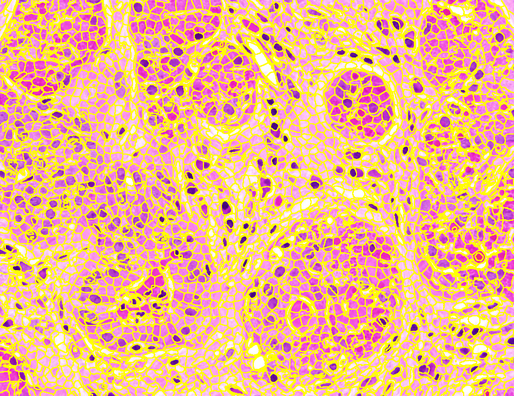}}
\hfill
\subfloat[S1/S2, SAM\label{fig_sub:he2_out_SAM}]{\includegraphics[width=0.249\linewidth]{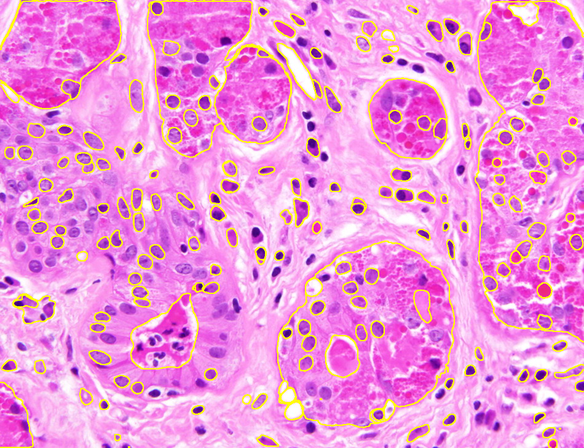}}

\subfloat[S1, AR\label{fig_sub:he2_out1_AR}]{\includegraphics[width=0.249\linewidth]{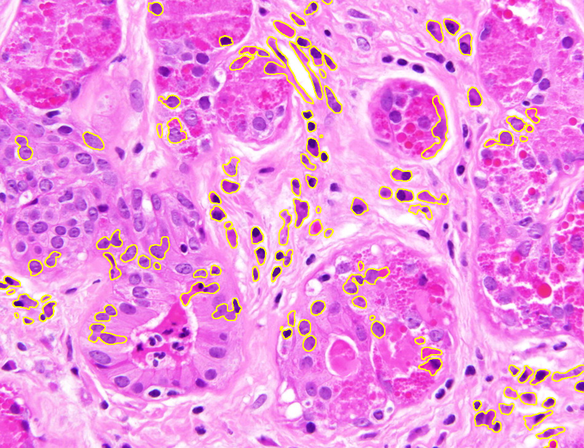}}
\hfill
\subfloat[S1, ZZ\label{fig_sub:he2_out1_ZZ}]{\includegraphics[width=0.249\linewidth]{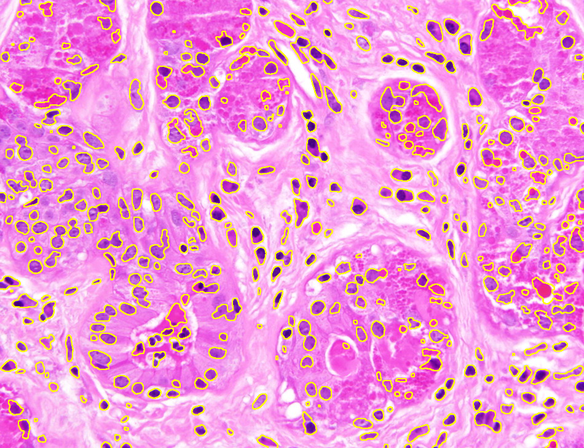}}
\hfill
\subfloat[S1, SMST\label{fig_sub:he2_out1_MST}]{\includegraphics[width=0.249\linewidth]{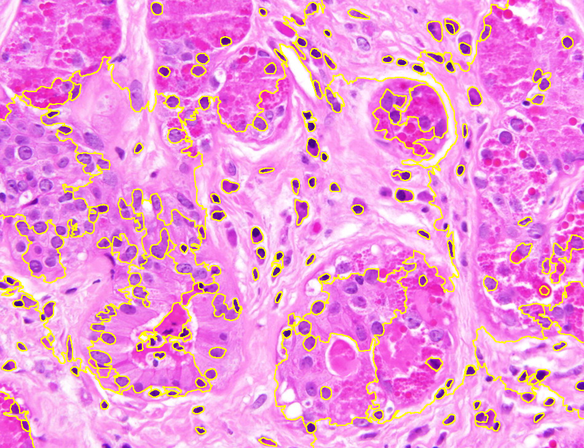}}
\hfill
\subfloat[S1, SP\label{fig_sub:he2_out1_SP}]{\includegraphics[width=0.249\linewidth]{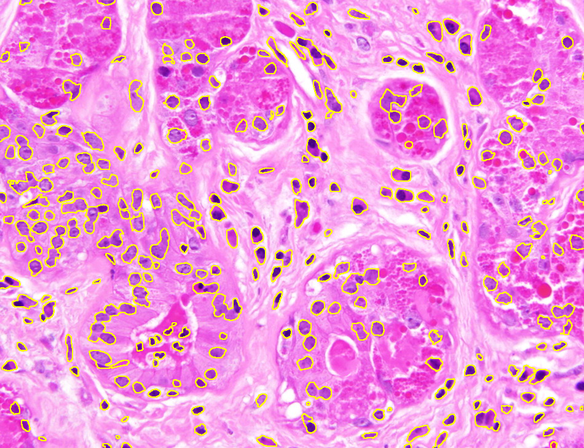}}

\subfloat[S2, AR\label{fig_sub:he2_out2_AR}]{\includegraphics[width=0.249\linewidth]{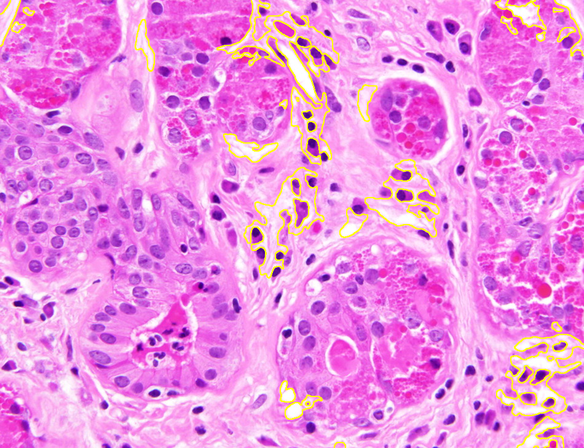}}\hfill
\subfloat[S2, ZZ\label{fig_sub:he2_out2_ZZ}]{\includegraphics[width=0.249\linewidth]{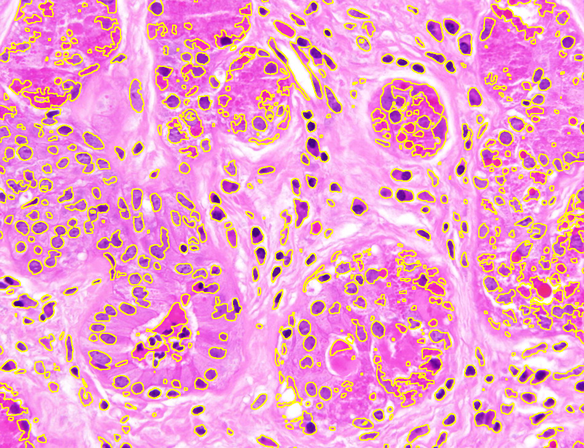}}
\hfill
\subfloat[S2, SMST\label{fig_sub:he2_out2_MST}]{\includegraphics[width=0.249\linewidth]{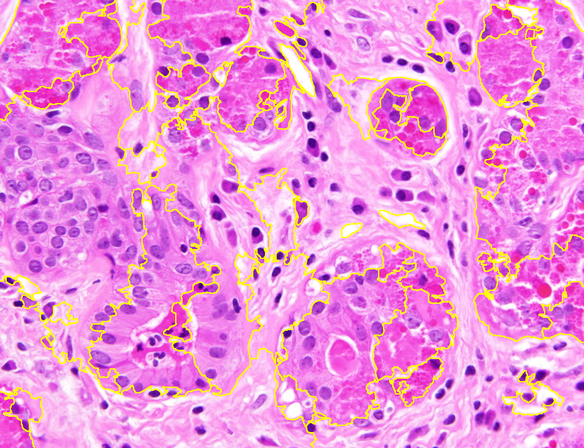}}
\hfill
\subfloat[S2, SP\label{fig_sub:he2_out2_SP}]{\includegraphics[width=0.249\linewidth]{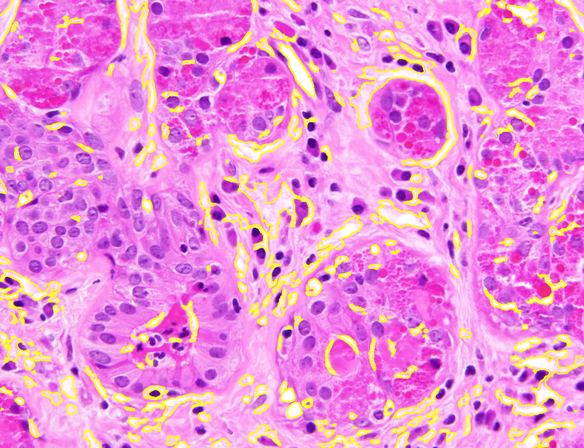}}

\caption{Unsupervised segmentation of the image in Fig.~\ref{fig_sub:he2_image} for two structures: purple cells (S1) and white cytoplasm (S2). Initializations via level sets (LS) for AR and ZZ are shown in (a, b); 6000 superpixels for SMST and SP in (c) (red channel for S1, remaining  channels for S2). SAM does not allow separate-channel segmentation; its overall output is in (d). Results for S1 are in (e–h) and for S2 in (i–l). Computation times (S1/S2, in seconds): AR 461.85/474.46, ZZ 2423.73/2434.59, SMST 0.42/0.38, SP 42.31/12.51; superpixels in (c) computed in 0.14s, SAM in (d) in 7.47s. 
}\label{fig:H&ESEG2}
\end{figure}

\subsection{Quantitative results}
We now assess SP using the 36 images from Dataset~2 and 50 images from Dataset~3. In our evaluation, we employ the Dice similarity coefficient (DSC) (as usual defined as the harmonic mean of precision and recall),
$\mathrm{DSC} = 2/(1/\mathrm{precision} + 1/\mathrm{recall}),$ $\mathrm{precision} = \mathrm{tp}/(\mathrm{tp} + \mathrm{fp}),$
$\mathrm{recall} = \mathrm{tp}/(\mathrm{tp} + \mathrm{fn}),$
where $\mathrm{tp}$ denotes the number of ground-truth points within the segmentation foreground, 
$\mathrm{fp}$ the number of predicted regions without ground truth, and 
$\mathrm{fn}$ the number of ground-truth points outside the segmentation foreground \cite{wienert2012detection}.  In Dataset~3, since pixel-level annotations were given~\cite{naylor2018segmentation}, the above metrics were calculated on a per-pixel basis.

To facilitate automatic comparison, we specify the parameters used for this task.  For AR and ZZ in Dataset~2, initial level sets were created by applying morphological dilation to the first~40 ground-truth points using a disk-shaped kernel with a radius of~5. If fewer than~40 points were available, all available points were dilated instead.  For  AR and ZZ in Dataset~3, pixel-level ground truth was provided, and the ground truth was eroded using a disk-shaped kernel with a radius of~5 as manually drawing these initializations would be time-consuming. SAM, SP, and SMST treat each connected component as a single object; they do not explicitly distinguish foreground from background. Consequently, the largest connected component was designated as background, with all remaining regions assigned to the foreground. Furthermore, because SMST often fragments the background into multiple regions, we treated SMST output regions containing holes as background. To implement this, regions with an Euler number less than 1 were removed using MATLAB’s \texttt{regionprops} function.

Table~\ref{tab:comparation_dataset_2}, on Dataset 2, shows that SP obtained the highest DSC (88.78), which is higher by 5.90, 7.09, 8.31, and 22.46 compared to SAM, ZZ, SMST, and AR, respectively. The segmentation time of SP was 14.55s, approximately twice that of SAM (the second fastest method without considering pre-processing times), and more than one order of magnitude lower than the level-set methods AR and ZZ. While AR and ZZ need user-specific markers as input (the pre-processing time indicated as $\infty$), SAM needs to load the pre-trained model (6.09s) while SP and SMST need to generate the superpixels in the pre-processing step (0.08s).  Figure~\ref{fig:H&ESEG3} shows the resulting segmentations for two example images from Figs.~\ref{fig_sub:he3_image} and \ref{fig_sub:he4_image} from Dataset~2. The segmentations reflect the trends observed in Table~\ref{tab:comparation_dataset_2}, with SP achieving the highest DSC and segmentation times exhibiting a similar relative order. 

\begin{table}
    \centering
    \begin{tabular}{llccccc}
    \toprule
                      &       &AR   &  SMST             &  ZZ  & SAM   &  SP \\ \midrule
{\small Dataset 2} &&&&&\\
&DSC                          &$66.32$                 &  $80.47$         &  $81.69$                  & $82.88$                         &   $\mathbf{88.78}$    \\
&Preproc. Time (s)           &$\infty$                &  $\mathbf{0.08}$ &  $\infty$                 & $6.09$                         &   $\mathbf{0.08}$ \\
&Segm.  Time  (s)          &$395.81$                &  $\mathbf{0.38}$ &  $2212.70$                & $7.75$                          &   $14.55$ \\ \midrule
{\small Dataset 3}&&&&&\\
&DSC                          &$61.35$                 &  $60.63$         &  $65.33$                  & $73.00$                         &   $\mathbf{76.62}$    \\
&Preproc. Time (s)         &$\infty$                &  $\mathbf{0.08}$ &  $\infty$                 & $6.65$                         &   $\mathbf{0.08}$ \\
&Segm.  Time (s)         &$302.76$                &  $\mathbf{0.39}$ &  $1758.84$                & $7.54$                          &   $16.47$ \\
\bottomrule
    \end{tabular}
    \caption{Quantitative comparison of segmentation accuracy (Dice coefficient, DSC) and running time on Dataset~2 and~3.}
    \label{tab:comparation_dataset_2}
\end{table}


On Dataset 3, the resulting DSC values, shown in Table~\ref{tab:comparation_dataset_2}, are  lower than those for Dataset~2. This difference is mainly caused by the fact that Dataset~3 provides pixel-level annotations, whereas Dataset~2 only includes nuclei center annotations. The computation times for AR and ZZ are shorter  due to the smaller image size ($512 \times 512$ vs. $600 \times 600$ in Dataset~2). SP achieves a DSC improvement of 3.62 over SAM, despite being roughly twice as slow as SAM in the segmentation process. The results also show that SP achieves again the highest DSC performance (76.62), outperforming SMST (DSC 60.63), AR (DSC 61.35), and ZZ (DSC 65.33). 

\begin{figure}[htb]
\centering
\subfloat[S3, SAM\\\hspace*{0.5ex}($7.62$/$81.88$)\label{fig_sub:he3_out_SAM}]{\includegraphics[width=0.199\linewidth]{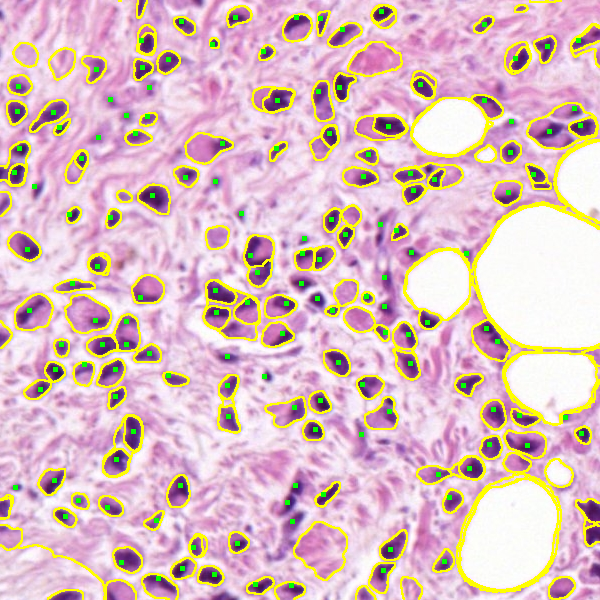}}
\hfill
\subfloat[S3, AR\\\hspace*{0.5ex}($234.69$/$64.32$)\label{fig_sub:he3_out_AR}]{\includegraphics[width=0.199\linewidth]{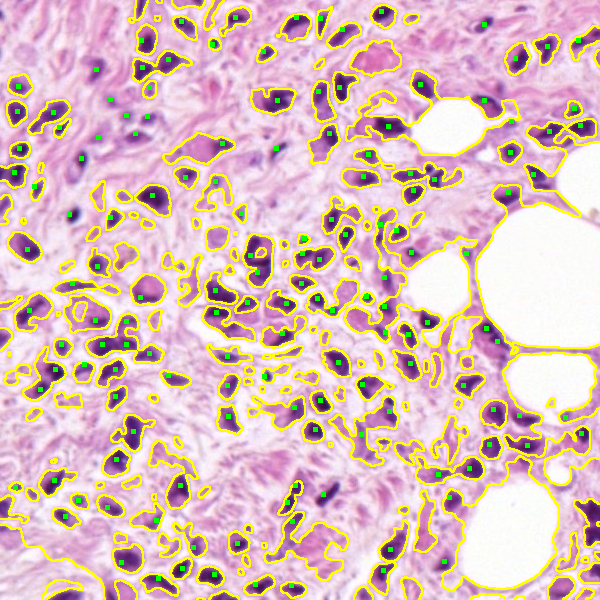}}
\hfill
\subfloat[S3, ZZ\\\hspace*{0.5ex}($1195.00$/$70.69$)\label{fig_sub:he3_out_ZZ}]{\includegraphics[width=0.199\linewidth]{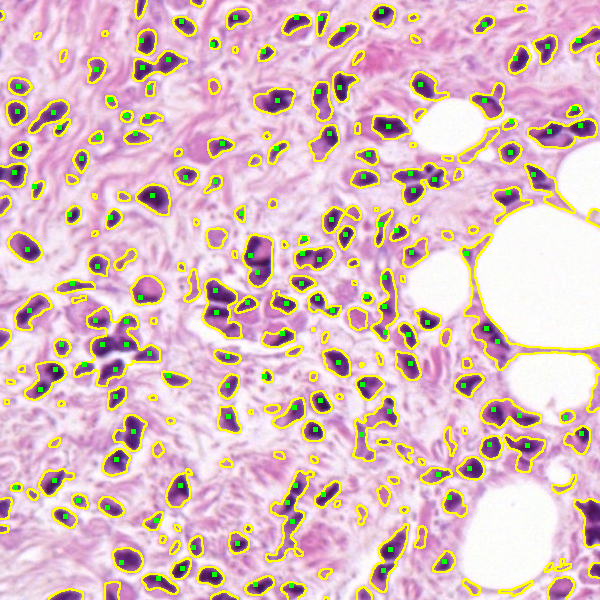}}
\hfill
\subfloat[S3, SMST\\\hspace*{0.5ex}($0.43$/$82.39$)\label{fig_sub:he3_out_MST}]{\includegraphics[width=0.199\linewidth]{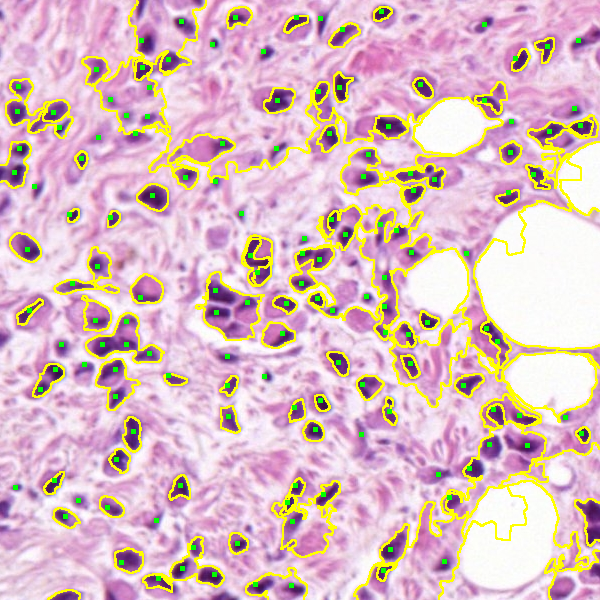}}
\hfill
\subfloat[S3, SP \\\hspace*{0.5ex}($4.32$/$86.77$)\label{fig_sub:he3_out_SP}]{\includegraphics[width=0.199\linewidth]{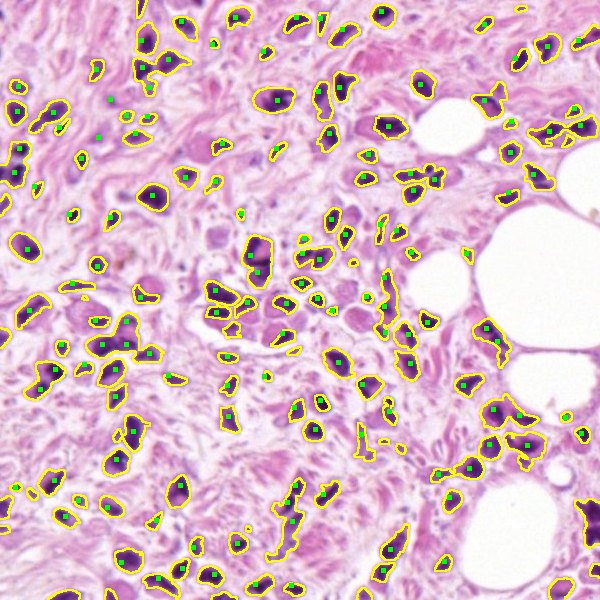}}

\subfloat[S4, SAM\\\hspace*{0.5ex}($7.48$/$83.64$)\label{fig_sub:he4_out_SAM}]{\includegraphics[width=0.199\linewidth]{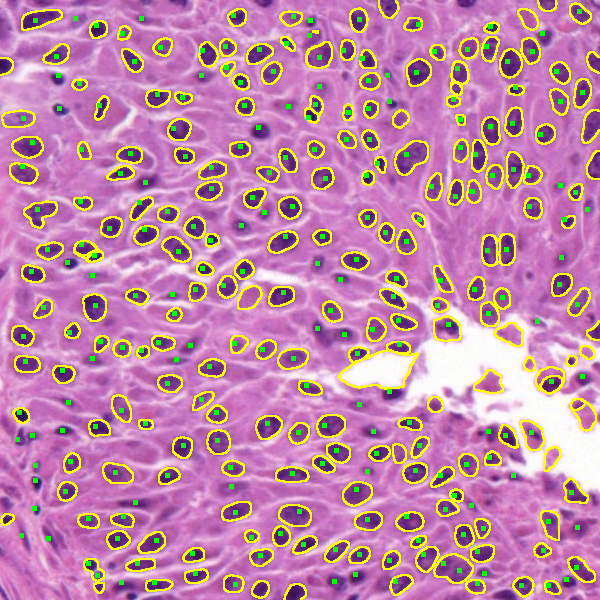}}
\hfill
\subfloat[S4, AR\\\hspace*{0.5ex}($428.36$/$76.72$)\label{fig_sub:he4_out_AR}]{\includegraphics[width=0.199\linewidth]{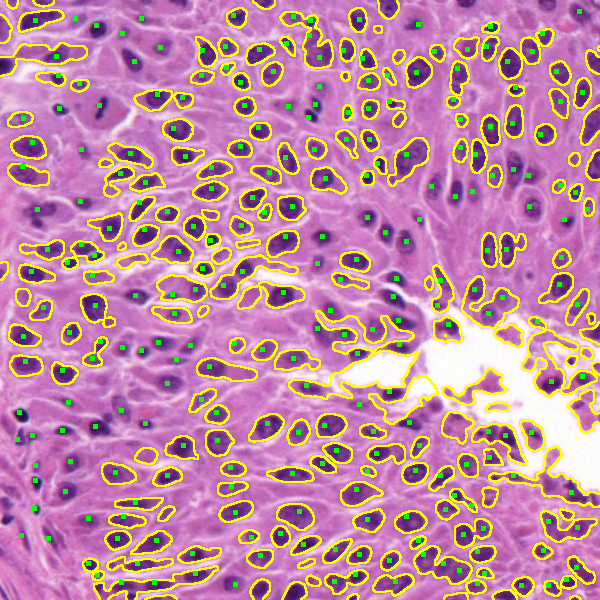}}
\hfill
\subfloat[S4, ZZ\\\hspace*{0.5ex}($2232.16$/$84.36$)\label{fig_sub:he4_out_ZZ}]{\includegraphics[width=0.199\linewidth]{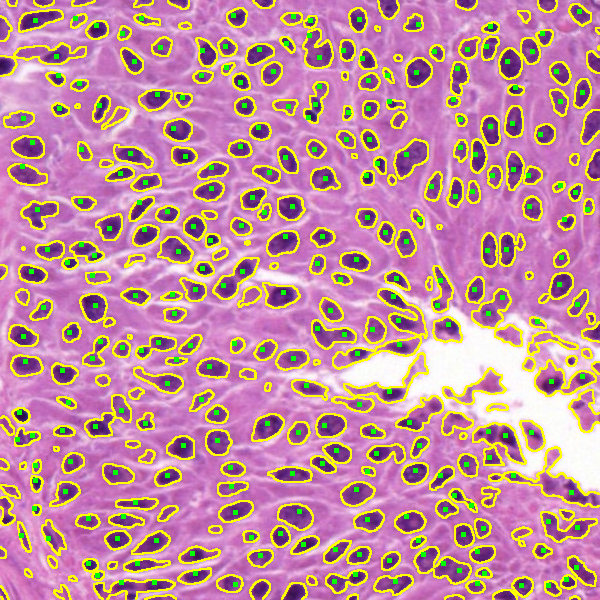}}
\hfill
\subfloat[S4, SMST\\\hspace*{0.5ex}($0.46$/$79.31$)\label{fig_sub:he4_out_MST}]{\includegraphics[width=0.199\linewidth]{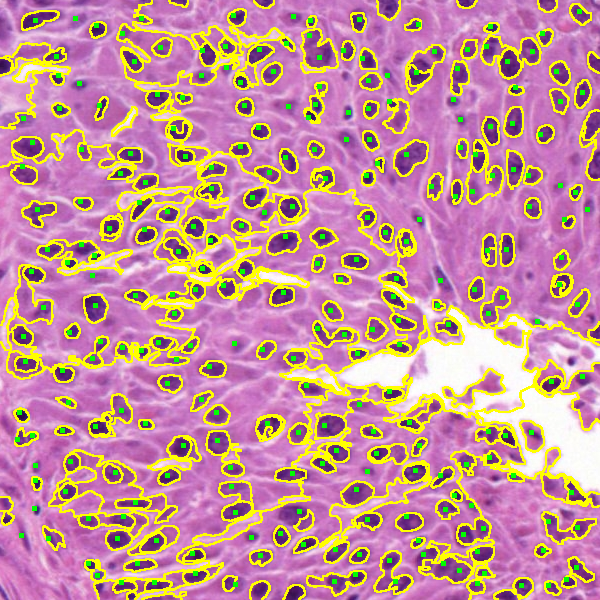}}
\hfill
\subfloat[S4, SP \\\hspace*{0.5ex}($15.75$/$89.33$)\label{fig_sub:he4_out_SP}]{\includegraphics[width=0.199\linewidth]{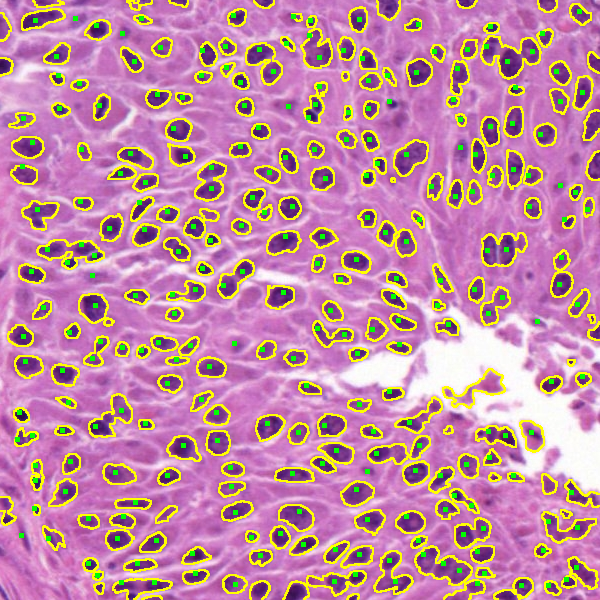}}

\caption{Segmentation of two  images from the quantitative comparison on Dataset~2 (S3 and S4, Figs. \ref{fig_sub:he3_image} and \ref{fig_sub:he4_image}); brackets show segmentation time (s) / DSC.}\label{fig:H&ESEG3}
\end{figure}

\section{Conclusion}\label{sec:conclusion} Global metrics, such as the variance-based energies in the Chan--Vese model~\cite{chan2001active} and the SMST framework~\cite{wang2017optimal}, effectively suppress local noise but remain sensitive to global intensity inhomogeneities. In contrast, local region-based models, including AR~\cite{ali2018image} and ZZ~\cite{zhang2015level}, improve robustness to inhomogeneity at the cost of increased computational complexity and a limited global receptive field.
We proposed a hierarchical segmentation approach that combines superpixel representations with OT-based region comparison. The two-stage structure enables a consistent use of distributional distances across scales, while reducing the complexity of region interactions and preserving meaningful boundaries. This formulation seems particularly well suited to images with strong intensity inhomogeneity.

Experiments show that the proposed method achieves higher segmentation accuracy than the considered classical and variational approaches, while being up to an order of magnitude faster than level-set methods. Its runtime remains competitive with fast deep-learning-based approaches such as SAM. Segmenting objects with weak boundaries or intensity distributions close to the background remains a challenge.

\section*{Acknowledgments}

This work was partially supported by the National Key R\&D Program of China 2021YFA1003003; the National Natural Science Foundation of China 61936002, T2225012;  the program of China Scholarships Council 202307300007; and the Engineering and Physical Sciences Research Council (EPSRC grant EP/X035883/1).

\section*{Declarations}

\begin{itemize}
\item \textbf{Conflict of interest} 
The authors declare no conflict of interest.
\item \textbf{Data availability} The datasets analyzed in this study are publicly available and cited in the manuscript. All images generated in the course of this work will be provided as supplementary material.
\item \textbf{Author contributions} J.H. developed the theory, implemented the code,
performed the computations and simulations, and wrote the first draft of the manuscript.
A.A., K.C., and N.L. guided, supervised, and reviewed the findings of this work.
All authors conceived the presented idea, discussed the results, and
contributed to the final manuscript.
\item \textbf{Use of generative-AI tools declaration}
The authors used OpenAI Codex solely for grammatical polishing of the manuscript and conversion to the AIMS LaTeX template. All theoretical derivations and all MATLAB and C++ code were developed manually by the authors. The authors take full responsibility for the content of the manuscript.
\end{itemize}

\bigskip

\appendix

\section{Complexity of global SP optimization}
\label{app:complexity}

We analyze the computational complexity of the global SP optimization problem associated with the SP model introduced in
Section~\ref{sec:segmodels}, taking the region-adjacency graph
$G=(V,E)$ and the associated histograms as inputs.
We restrict attention to two-bin histograms
represented by nonzero count vectors
$C_i\in\mathbb Z_{\geq0}^2$, encoded in binary,
with $v_{\Omega_i}=C_i/\|C_i\|_1$.
For a region obtained by successive merges of
superpixels indexed by $S\subseteq V$, its
normalized histogram is
$v_{\Omega_S}=(\sum_{i\in S}C_i)/
\|\sum_{i\in S}C_i\|_1$.
The two bins are located at $0$ and $1$. We write $C=(C_i)_{i\in V}$ for the family of count vectors.

For the global minimization problem, write
\[
\operatorname{OPT}_n(G,C)
=
\min_{\mathcal M\in\mathfrak A_{m\to n}}
\mathcal E_{\mathrm{reg}}(\mathcal M).
\]
We consider the following zero-threshold decision problem:
given a region-adjacency graph $G=(V,E)$ and
binary-encoded histogram count vectors
$C=(C_i)_{i\in V}$, decide whether
$\operatorname{OPT}_2(G,C)\leq0$.

\begin{lemma}\label{lem:complexity-two}
Let $W\in\mathbb{Z}_{>0}$, and suppose that each vertex $i$ is assigned
an integer $\omega_i$ with $|\omega_i|\leq W$ and histogram
$C_i=(W-\omega_i,W+\omega_i)$. For nonempty disjoint sets $A,B\subseteq V$,
\[
 \operatorname{d}^2(v_{\Omega_A},v_{\Omega_B})
 =\frac{|\overline{\omega}_A-\overline{\omega}_B|}{2W},\qquad
 \overline{\omega}_S=\frac{1}{|S|}\sum_{i\in S}\omega_i.
\]
\end{lemma}
\begin{proof} The second-bin mass of $v_{\Omega_S}$ is
$\tfrac12+\overline{\omega}_S/(2W)$.
Since, by the one-dimensional quantile formula
\cite[Eq.~2.36]{peyre2019computational}, the squared
$2$-Wasserstein distance between two probability
measures supported on $\{0,1\}$ equals the
absolute difference of their second-bin masses,
the assertion follows.
\end{proof}

\begin{lemma}\label{lem:complexity-tel}
If an admissible merge sequence $\mathcal M$
ends with regions $\Omega_1,\ldots,\Omega_n$,
then
\[\mathcal E_{\mathrm{reg}}(\mathcal M)
=\sum_{i=1}^n h_i\geq0.\]
In particular, the energy vanishes if and only if
every final region has zero heterogeneity.
\end{lemma}
\begin{proof}
At each merge, the regularized cost is the
heterogeneity of the merged region minus those
of its two constituent regions.
Summing over all merges cancels the heterogeneity
of every intermediate region. Since the initial
heterogeneities are zero, only those of the final
regions remain, and these are nonnegative.
\end{proof}

\begin{theorem}\label{thm:complexity-main}
The zero-threshold decision problem for the global SP
optimization problem with $n=2$ is $\mathsf{NP}$-complete,
even when $G$ has treewidth at most three and is
realizable as the region-adjacency graph of a partition
of a rectangular pixel image into equally sized,
connected superpixels, and the two-bin histograms
have equal total mass.
\end{theorem}
\begin{proof}
Membership in $\mathsf{NP}$ follows by providing a sequence of $m-2$
admissible merges. Admissibility, aggregated histograms,
heterogeneities, and the total energy can be verified exactly
using rational arithmetic with polynomially many bits.

For $\mathsf{NP}$-hardness, we reduce from \textsc{Subset Sum}
\cite{GareyJohnson1979}. Let $b_1,\ldots,b_\ell$ be positive
integers and let $0<T<U:=\sum_{i=1}^{\ell}b_i$. Set $L=\ell+2$ and $V_0=\{1,\ldots,L\}$, and assign
$\omega_i=b_i$, $i=1,\ldots,\ell$;
$\omega_{\ell+1}=-T$ and $\omega_{\ell+2}=T-U$. Then $\sum_{i\in V_0}\omega_i=0$. Set
$M=\max_{i\in V_0}|\omega_i|$.

Construct $G=(V,E)$ with
$V=V_0\cup\{\ell+3\}$ and
\[
\begin{aligned}
E={}&
\{\{i,\ell+1\},\{i,\ell+2\}:1\leq i\leq\ell\}\\
&{}\cup\{\{i,i+1\}:1\leq i<\ell\}
\cup\{\{\ell+1,\ell+3\}\}.
\end{aligned}
\]
A tree decomposition of width at most three
is obtained by arranging the vertex sets
$\{\ell+1,\ell+2,i,i+1\}$,
$i=1,\dots,\ell-1$, along a path and attaching
$\{\ell+1,\ell+3\}$ to one end.

Moreover, $G$ is the region-adjacency graph of an
$\ell\times(\ell+3)$ rectangular pixel image:
assign its first, second, and last columns to
vertices $\ell+3$, $\ell+1$, and $\ell+2$,
respectively, and the $\ell$ intermediate columns
in row $i$ to vertex $i$. Thus all $\ell+3$
superpixels are connected rectangles containing
exactly $\ell$ pixels, and the partition can be
constructed in polynomial time.

Set $\omega_{\ell+3}=2LM,$ $W=\omega_{\ell+3},$ $C_i=(W-\omega_i,W+\omega_i),$ for all $i\in V$. These are nonnegative integer histograms with common
total mass $2W$, and the construction has polynomial
encoding length.

We first show that any region containing $\ell+3$
and at least one other vertex has positive heterogeneity.
At its last merge, write the region as
$A\mathbin{\dot\cup}B$, with $\ell+3\in A$.
Since $B\subseteq V_0$, we have
$\overline{\omega}_B\leq M$. Moreover, $|A|\leq L$,
as $B$ is nonempty, and hence
\[
\overline{\omega}_A
\geq\left(W-(|A|-1)M\right)/|A|
>M\geq\overline{\omega}_B.
\]
The strict inequality follows from
$W=2LM>(2|A|-1)M$.
By Lemma~\ref{lem:complexity-two}, the resulting
heterogeneity is positive.

Suppose the \textsc{Subset Sum} instance has a
solution $I\subseteq\{1,\ldots,\ell\}$.
The sets
\[
A=\{\ell+1\}\cup I,\qquad
B=\{\ell+2\}\cup
(\{1,\ldots,\ell\}\setminus I)
\]
are connected in $G$, and their weights both sum
to zero. Since $A$ and $B$ are connected, we may build each of the regions $\Omega_A$ and $\Omega_B$ by successive merges of adjacent superpixels, and then merge these two regions, which are adjacent because $I\neq\emptyset$ (as $T>0$), so some $i\in I$ is adjacent to $\ell+2\in B$. By
Lemma~\ref{lem:complexity-two}, the resulting
region $V_0$ has zero heterogeneity.
Leaving $\ell+3$ as a singleton, we obtain
$\operatorname{OPT}_2(G,C)=0$ by
Lemma~\ref{lem:complexity-tel}.

Conversely, suppose an admissible merge sequence
ending in two regions has zero energy. By
Lemma~\ref{lem:complexity-tel}, both final
heterogeneities vanish. The preceding argument
implies that $\{\ell+3\}$ is a singleton final
region, hence the other is $V_0$.
Let $A,B$ be the two regions in the last merge
forming $V_0$. Lemma~\ref{lem:complexity-two}
gives $\overline{\omega}_A=\overline{\omega}_B$.
Since the weights in $V_0$ sum to zero,
$
\sum_{i\in A}\omega_i=\sum_{i\in B}\omega_i=0.
$ Each of $A$ and $B$ must contain one of the
vertices $\ell+1,\ell+2$, since all other
weights are positive. If $A$ contains $\ell+1$,
its item vertices have weights summing to~$T$.
Otherwise, $A$ contains $\ell+2$, and its item
vertices have weights summing to $U-T$;
their complement therefore sums to $T$.
Thus the \textsc{Subset Sum} instance has a
solution.

The reduction is polynomial, proving
$\mathsf{NP}$-completeness.
\end{proof}

Theorem~\ref{thm:complexity-main} implies that computing a globally optimal merge sequence is $\mathsf{NP}$-hard, even under the stated restrictions. Moreover, since $\mathcal E_{\mathrm{reg}}(\mathcal M)\geq0$ for every admissible merge sequence, no polynomial-time multiplicative approximation algorithm with a finite approximation guarantee exists for this problem unless $\mathsf{P}=\mathsf{NP}$. Indeed, such an algorithm would return a zero-energy solution whenever the optimum is zero, thereby deciding an $\mathsf{NP}$-complete problem.

\section{Running time of the algorithm~\ref{alg:algorithm}}
\label{app:coputational-cost}

\begin{proof}[Proof of Proposition~\ref{prop:computational-cost}] The algorithm consists of an initialization phase followed by a merging loop. 

Computing $k$ representative colors takes $T_{\mathrm{rep}}(N,k)$ time.
Superpixelization of $I$ via Power-SLIC, including histogram initialization,
requires $O(N)$ time, while initializing region heterogeneity levels takes
$O(m)$. For each adjacent region pair, one distance
$\operatorname{d}(\Omega_i,\Omega_j)$ is evaluated
and inserted into the priority queue, yielding
a total cost of
$O(E(T_{\mathrm{OT}}(k)+\log m))$.

The main loop performs~ $m-n$ merge operations. The worst case occurs for~$n=1,$
 which we therefore analyze. A single merge operation consists of merging the two regions as sets, merging their
adjacency information, and merging the associated histograms with $k$ bins.

When two regions are merged, they are removed, and a new, larger region is formed. Distance computations and priority queue insertions are
performed for all adjacent pairs involving the newly formed region and its $O(m)$ neighbors, yielding a total cost of $O(m^2\,T_{\mathrm{OT}}(k))$ for distance recomputations and
$O(m^2)$ priority queue insertions overall, each costing $O(\log(m)).$

Each merge incurs an $O(k)$ cost for histogram updates. As at most $m-1$ merges
are performed, the total cost of histogram updates is $O(mk)$.

Unsuccessful extractions from the priority queue discard stale entries and do
not affect the asymptotic running time.

Combining these bounds yields the stated worst-case running time.
\end{proof}

\section{Exact recovery under histogram separation}
\label{app:recovery}

We next identify a setting in which the greedy algorithm for the SP model recovers a
prescribed two-region segmentation and admits an additive
approximation guarantee. Unlike the preceding complexity
analysis, the result is not restricted to two-bin histograms. Here $C=(C_i)_{i\in V}$ denotes the family of
nonzero histogram count vectors with an arbitrary
common number of bins, with aggregated histograms
defined as in the preceding section.  In this section, $\Omega_1,\ldots,\Omega_m$
denote the initial superpixels, and for $S\subseteq V$
we write $\Omega_S=\bigcup_{i\in S}\Omega_i$.

\begin{theorem}\label{thm:recovery}
Suppose that $V=V_A\mathbin{\dot\cup}V_B$ is a partition
into two nonempty sets, each inducing a connected
subgraph of $G$, and that there exist normalized
histograms $v_A^*,v_B^*$ and $r\geq0$ such that
$d(v_{\Omega_i},v_X^*)\leq r$ for all
$i\in V_X$, $X\in\{A,B\}$.
If $\operatorname{d}(v_A^*,v_B^*)>(2+2\sqrt{3})r,$ 
then the greedy algorithm for the SP model, stopped at two regions,
returns the partition $\{\Omega_{V_A},\Omega_{V_B}\}$,
independently of tie-breaking, and its merge sequence $\mathcal M_{\mathrm{greedy}}$ satisfies
$
0\leq
\mathcal E_{\mathrm{reg}}(\mathcal M_{\mathrm{greedy}})
-\operatorname{OPT}_2(G,C)
\leq 8r^2.
$
\end{theorem}

\begin{proof}
Write $w_i=\|C_i\|_1.$ For every nonempty
$R\subseteq V_X$, $X\in\{A,B\}$, we have
$v_{\Omega_R}=\sum_{i\in R}\lambda_i v_{\Omega_i}$,
where $\lambda_i=w_i/\sum_{j\in R}w_j$.
By convexity of the squared Wasserstein distance
with respect to mixtures (see, e.g.,~\cite[Sect.~9.1]{peyre2019computational}),
it follows that
$\operatorname{d}(v_{\Omega_R},v_X^*)\leq r$.  
Hence, by the triangle inequality,
$\operatorname{d}(v_{\Omega_R},v_{\Omega_S})\leq 2r$
for nonempty $R,S\subseteq V_X$, $X\in\{A,B\}$,
whereas
$\operatorname{d}(v_{\Omega_R},v_{\Omega_S})\geq\Delta-2r$
for nonempty $R\subseteq V_A$ and $S\subseteq V_B$,
where $\Delta=\operatorname{d}(v_A^*,v_B^*)$.

We show inductively that every greedy merge
is within one class. 
If all previous merges were within-class,
then every current region is contained in one
class and has heterogeneity at most $4r^2$:
initial regions have zero heterogeneity, while
each merged region has heterogeneity
$\operatorname{d}^2(v_{\Omega_R},v_{\Omega_S})\leq 4r^2$,
where~$R$ and~$S$ are the two regions from
which it was formed.

Consequently, the regularized merge cost
$\kappa_{ij}$ of any admissible within-class
pair of current regions $\Omega_i,\Omega_j$
satisfies $\kappa_{ij}\leq4r^2$, whereas
every admissible cross-class pair satisfies
$\kappa_{ij}\geq(\Delta-2r)^2-8r^2>(2\sqrt{3}r)^2-8r^2=4r^2$
by the separation assumption. 

As long as more than two regions remain,
connectivity of the class-induced subgraphs
guarantees an admissible within-class pair.
Thus every greedy merge is within-class,
and the final regions are
$\Omega_{V_A}$ and $\Omega_{V_B}$. Their heterogeneities are at most $4r^2$
each. The additive bound follows from
Lemma~\ref{lem:complexity-tel} and
$\operatorname{OPT}_2(G,C)\geq0$.
\end{proof}

In particular, if $r=0$ and $v_A^*\neq v_B^*$,
the greedy algorithm for the SP model recovers the two homogeneous
histogram classes exactly and achieves the
global minimum
$\mathcal E_{\mathrm{reg}}(\mathcal M_{\mathrm{greedy}})
=\operatorname{OPT}_2(G,C)=0$.

\bibliographystyle{AIMS}
\bibliography{sample}

\medskip
Received xxxx 20xx; revised xxxx 20xx; early access xxxx 20xx.
\medskip

\end{document}